\RequirePackage{fix-cm}
\documentclass[smallcondensed]{svjour3}       % onecolumn (second format)
\smartqed  % flush right qed marks, e.g. at end of proof
\newtheorem{mydef}{Definition}

\usepackage[ruled,vlined]{algorithm2e}
\usepackage{multirow}
\usepackage{mathrsfs}
\usepackage{amssymb}
\usepackage[cmex10]{amsmath}
\usepackage[pdftex]{graphicx}
\usepackage{balance}
\usepackage{cite}
\usepackage[numbers]{natbib}
\usepackage{wrapfig}
\usepackage{hyperref}

\begin{document}

\title{Preserving Differential Privacy in Convolutional Deep Belief Networks%\thanks{Grants or other notes
%about the article that should go on the front page should be
%placed here. General acknowledgments should be placed at the end of the article.}
}

%\titlerunning{Short form of title}        % if too long for running head

\author{NhatHai Phan$^1$\thanks{$^1$ This is a correction version of the previous arXiv:1706.08839 and ML'17 published version. Refer to Appendix A for summary of changes.}        \and
        Xintao Wu \and Dejing Dou %etc.
}

\authorrunning{Accepted by Machine Learning, 2017 (Journal Track of ECML-PKDD)} % if too long for running head

\institute{N. Phan \at
              New Jersey Institute of Technology, USA \\
              Tel.: +1-973-5966367\\
              \email{phan@njit.edu}
           \and
           X. Wu \at
              University of Arkansas, USA \\
              Tel.: +1-479-5756519\\
              \email{xintaowu@uark.edu} 
           \and 
           D. Dou \at
              University of Oregon, USA \\
              Tel.: +1-541-3464572\\
              \email{dou@cs.uoregon.edu} 
}

%\date{Received: 18 Nov 2016 / Accepted: 24 Jun 2017}
% The correct dates will be entered by the editor

\maketitle

\begin{abstract}
The remarkable development of deep learning in medicine and healthcare domain presents obvious privacy issues, when deep neural networks are built on users' personal and highly sensitive data, e.g., clinical records, user profiles, biomedical images, etc. However, only a few scientific studies on preserving privacy in deep learning have been conducted. In this paper, we focus on developing a private convolutional deep belief network (pCDBN), which essentially is a convolutional deep belief network (CDBN) under differential privacy. Our main idea of enforcing $\epsilon$-differential privacy is to leverage the functional mechanism to perturb the energy-based objective functions of traditional CDBNs, rather than their results. One key contribution of this work is that we propose the use of Chebyshev expansion to derive the approximate polynomial representation of objective functions. Our theoretical analysis shows that we can further derive the sensitivity and error bounds of the approximate polynomial representation. As a result, preserving differential privacy in CDBNs is feasible. We applied our model in a health social network, i.e., YesiWell data, and in a handwriting digit dataset, i.e., MNIST data, for human behavior prediction, human behavior classification, and handwriting digit recognition tasks. Theoretical analysis and rigorous experimental evaluations show that the pCDBN is highly effective. It significantly outperforms existing solutions. 
% \PACS{PACS code1 \and PACS code2 \and more}
% \subclass{MSC code1 \and MSC code2 \and more}
\end{abstract}

\section{Introduction} 
Today, amid rapid adoption of electronic health records and wearables, the global health care systems are systematically collecting longitudinal patient health information, e.g., diagnoses, medication, lab tests, procedures, demography, clinical notes, etc. The patient health information is generated by one or more encounters in any healthcare delivery systems \cite{PMID:26828707}. Healthcare data is now measured in exabytes, and it will reach the zettabyte and the yottabyte range in the near future \cite{Fang:2016:CHI:2911992.2932707}. Although appropriate in a variety of situations, many traditional methods of analysis do not automatically capture complex and hidden features from large-scale and perhaps unlabeled data \cite{citeulike:14040136}. In practice, many health applications depend on including domain knowledge to construct relevant features, some of which are further based on supplemental data. This process is not straightforward and time consuming. That may result in missing opportunities to discover novel patterns and features. %Many of traditional methods of analysis are not sufficient to take advantage of the potential that large-scale healthcare data holds, especially in the age of digital health. 

This is where \textit{deep learning}, which is one of the state-of-the-art machine learning techniques, comes in to take advantage of the potential that large-scale healthcare data holds, especially in the age of digital health. Deep neural networks can discover novel patterns and dependencies in both unlabeled and labeled data by applying state-of-the-art training algorithms, e.g., greedy-layer wise \cite{Hinton2006}, contrastive divergent algorithm \cite{Hinton2002}, etc. That makes it easier to extract useful information when building classifiers and predictors \cite{citeulike:13629676}. 

Deep learning has applications in a number of healthcare areas, e.g., phenotype extraction and health risk prediction \cite{Feiwang:2016}, prediction of the development of various diseases including schizophrenia, a variety of cancers, diabetes, heart failure, etc. \cite{Choiocw112,Li:2015:PIR:2817095.2817104,citeulike:14040136,Roumia2014,citeulike:7685411}, prediction of risk of readmission \cite{citeulike:7685411}, Alzheimer's diagnosis \cite{DBLP:conf/isbi/LiuLCPKF14,doi:10.1142/S0129065716500258}, risk prediction for chronic kidney disease progression \cite{Perotte872}, physical activity prediction \cite{PhanAsonam2015,PhanDPK16,Phan:2015:ODL,Phan0WD16}, feature learning from fMRI data \cite{10.3389/fnins.2014.00229}, diagnosis code assignment \cite{Gottlieb2013,Perotte231}, reconstruction of brain circuits \cite{helmstaedter2013}, prediction of the activity of potential drug molecules \cite{DBLP:MaSLDS15}, the effects of mutations in non-coding DNA on gene expressions \cite{Leung15062014,Xiong1254806}, and many more.

The development of deep learning in the domain of medicine and healthcare presents obvious privacy issues, when deep neural networks are built based on patients' personal and highly sensitive data, e.g., clinical records, user profiles, biomedical images, etc. To convince individuals to allow that their data be included in deep learning projects, principled and rigorous privacy guarantees must be provided. However, only a few deep learning techniques have yet been developed that incorporate privacy protections. In clinical trials, such lack of protection and efficacy may put patient data at high risk and expose healthcare providers to legal action based on HIPAA/HITECH law \cite{HIPAA,HITECH}. Motivated by this, we aim to develop an algorithm to preserve privacy in fundamental deep learning models in this paper.

%In recent years, the techniques developed from deep learning research have already been impacting a wide range of real-world applications including signal and information processing \cite{Report:209355}, bioinformatics~\cite{Chicco:2014}, medicine and healthcare \cite{Song:2014:781}, and many more. When those techniques have been built on datasets, which may contain highly sensitive information, this has presented clear privacy issues. To convince individuals to allow deep learning based on such data, principled and rigorous privacy guarantees must be provided.

%Releasing sensitive results of statistical analysis and data mining while protecting privacy has been studied in the past few decades. One state-of-the-art approach to the problem is $\epsilon$-differential privacy \cite{dwork2006calibrating}, which works by injecting random noise into the released statistical results computed from the underlying sensitive data, such that the distribution of the noisy results is relatively insensitive to any change of a single record in the original dataset. This ensures that the adversary cannot infer any information about any particular record with high confidence (controlled by parameter $\epsilon$), even if the adversary possesses all the remaining tuples of the sensitive data. 

Releasing sensitive results of statistical analyses and data mining while protecting privacy has been studied in the past few decades. One state-of-the-art privacy model is $\epsilon$-differential privacy \cite{dwork2006calibrating}. A differential privacy model ensures that the adversary cannot infer any information about any particular data record with high confidence (controlled by a privacy budget $\epsilon$) from the released learning models. This strong standard for privacy guarantees is still valid, even if the adversary possesses all the remaining tuples of the sensitive data. The privacy budget $\epsilon$ controls the amount by which the output distributions induced by two neighboring databases may differ. We say that two databases are neighboring if they differ in a single data record, that is, if one data record is present in one database and absent in the other. It is clear that the smaller values of $\epsilon$ enforce a stronger privacy guarantee. This is because it is more difficult to infer any particular data record by distinguishing any two neighboring databases from the output distributions. Differential privacy research has been studied from the theoretical perspective, e.g., \cite{chaudhuri2008privacy,hay2010boosting,DBLP:conf/sigmod/KiferM11,DBLP:conf/kdd/LeeC12}.
Different types of mechanisms (e.g., the Laplace mechanism \cite{dwork2006calibrating}, the smooth sensitivity \cite{nissim2007smooth}, the exponential mechanism \cite{DBLP:conf/focs/McSherryT07}, and the perturbation of objective function \cite{chaudhuri2008privacy}) have been studied to enforce differential privacy.

Combining differential privacy and deep learning, i.e., the two state-of-the-art techniques in privacy preserving and machine learning, is timely and crucial. This is a non-trivial task, and therefore only a few scientific studies have been conducted. In \cite{ShokriVitaly2015}, the authors proposed a distributed training method, which directly injects noise into gradient descents of parameters, to preserve privacy in neural networks. The method is attractive for applications of deep learning on mobile devices. However, it may consume an unnecessarily large portion of the privacy budget to ensure model accuracy, as the number of training epochs and the number of shared parameters among multiple parties are often large. To improve this, based on the composition theorem \cite{Dwork:2009:DPR}, Abadi et al. \cite{Abadi} proposed a privacy accountant, which keeps track of privacy spending and enforces applicable privacy policies. The approach is still dependent on the number of training epochs. With a small privacy budget $\epsilon$, only a small number of epochs can be used to train the model. In practice, that could potentially affect the model utility, when the number of training epochs needs to be large to guarantee the model accuracy. 

Recently, Phan et al. \cite{Phan0WD16} proposed deep private auto-encoders (dPAs), in which differential privacy is enforced by perturbing the objective functions of deep auto-encoders \cite{Bengio2009}. It is worthy to note that the privacy budget consumed by dPAs is independent of the number of training epochs. A different method, named \textbf{CryptoNets}, was proposed in \cite{pmlr-v48-gilad-bachrach16} towards the application of neural networks to encrypted data. A data owner can send their encrypted data to a cloud service that hosts the network, and get encrypted predictions in return. This method is different from our context, since it does not aim at releasing learning models under privacy protections.

Existing differential privacy preserving algorithms in deep learning pose major concerns about their applicability. They are either designed for a specific deep learning model, i.e., deep auto-encoders \cite{Phan0WD16}, or they are affected by the number of training epochs \cite{ShokriVitaly2015,Abadi}. Therefore, there is an urgent demand for the development of a privacy preserving framework, such that: \textbf{(1)} It is totally independent of the number of training epochs in consuming privacy budget; and \textbf{(2)} It has the potential to be applied in typical energy-based deep neural networks. Such frameworks will significantly promote the application of privacy preservation in deep learning. 

Motivated by this, we aim at developing a \textit{private convolutional deep belief network} (pCDBN), which essentially is a convolutional deep belief network (CDBN) \cite{Lee:2009} under differential privacy. CDBN is a typical and well-known deep learning model. It is an energy-based model. Preserving differential privacy in CDBNs is non-trivial, since CDBNs are more complicated compared with other fundamental models, such as auto-encoders and Restricted Boltzmann Machines (RBM) \cite{Smolensky1986}, in terms of structural designs and learning algorithms. In fact, there are multiple groups of hidden units in each of which parameters are shared in a CDBN. Inappropriate analysis might result in consuming too much of a privacy budget in training phases. The privacy consumption also must be independent of the number of training epochs to guarantee the potential to work with large datasets.

Our key idea is to apply Chebyshev Expansion \cite{Chebyshev} to derive polynomial approximations of non-linear objective functions used in CDBNs, such that the design of differential privacy-preserving deep learning is feasible. Then, we inject noise into these polynomial forms, so that the $\epsilon$-differential privacy is satisfied in the training phases of each hidden layer by leveraging \textit{functional mechanism} \cite{zhang2012functional}. Third, hidden layers now become private hidden layers, which can be stacked on each other to produce a \textit{private convolutional deep belief network} (pCDBN). 

To demonstrate the effectiveness of our framework, we applied our model for binomial human behavior prediction and classification tasks in a health social network. A novel human behavior model based on the pCDBN is proposed to predict whether an overweight or obese individual will increase physical exercise in a real health social network. To illustrate the ability to work with large-scale datasets of our model, we also conducted additional experiments on the well-known handwriting digit dataset (MNIST data) \cite{Lecun726791}. We compare our model with the private stochastic gradient descent algorithm, denoted \textbf{pSGD}, from Albadi et al. \cite{Abadi}, and the deep private auto-encoders (\textbf{dPAs}) \cite{Phan0WD16}. The pSGD and dPAs are the state-of-the-art algorithms in preserving differential privacy in deep learning. Theoretical analysis and rigorous experimental evaluations show that our model is highly effective. It significantly outperforms existing solutions. 

The rest of the paper is organized as follows. In Section 2, we introduce preliminaries and related works. We present our private convolutional deep belief network in Section 3. The experimental evaluation is in Section 4, and we conclude the paper in Section 5.

\section{Preliminaries and Related Works} 
In this section, we briefly revisit the definition of differential privacy, functional mechanism \cite{zhang2012functional}, convolutional deep belief networks \cite{Lee:2009}, and the Chebyshev Expansion \cite{Chebyshev}.
Let $D$ be a database that contains $n$ tuples $t_1, t_2, \ldots, t_n$ and
$d$+$1$ attributes $X_1, X_2, \ldots, X_d, Y$. For each tuple $t_i = (x_{i1}, x_{i2}, \ldots, x_{id}, y_i)$, we assume, without loss of generality,
%~\footnote{This assumption can be easily enforced by changing each $x_{ij}$ to $\frac{x_{ij}-\alpha_j}{(\beta_j-\alpha_j)\cdot\sqrt{d}}$, where $\alpha_j$ and $\beta_j$ denote the minimum and maximum values in the domain of $X_j$.}
 $\sqrt{\sum^d_{j=1} x^2_{ij}} \leq 1$ where $x_{ij} \geq 0$, $y_i$ follows a binomial distribution.
Our objective is to construct a deep neural network $\rho$ from $D$ that (i) takes $\mathbf{x}_i = (x_{i1}, x_{i2}, \ldots, x_{id})$ as input and (ii) outputs a prediction of $y_i$ that is as accurate as possible. $t_i$ and $\mathbf{x}_i$ are used exchangeably to indicate the data tuple $i$. The model function $\rho$ contains a model parameter vector  $W$. To evaluate whether $W$ leads to an accurate model, a cost function $f_D(W)$ is often used to measure the difference between the original and predicted values of $y_i$. As the released model parameter $W$ may disclose sensitive information of $D$, to protect the privacy, we require that the model training should be performed with an algorithm that satisfies $\epsilon$\textit{-differential privacy}. 

Differential privacy \cite{dwork2006calibrating} establishes a strong standard for privacy guarantees for algorithms, e.g., training algorithms of machine learning models, on aggregate databases. It is defined in the context of neighboring databases. We say that two databases are neighboring if they differ in a single data record. That is, if one data record is present in one database and absent in the other. The definition of differential privacy is as follows: 

\begin{mydef}{($\epsilon$-Different Privacy \cite{dwork2006calibrating}).} A randomized algorithm $A$ fulfills $\epsilon$-differential privacy, iff for any two databases $D$ and $D'$ differing at most one tuple, and for all $O \subseteq Range(A)$, we have:
\begin{equation}
Pr[A(D) = O] \leq e^\epsilon Pr[A(D') = O] 
\end{equation}
where the privacy budget $\epsilon$ controls the amount by which the distributions induced by two neighboring datasets may differ. Smaller values of $\epsilon$ enforce a stronger privacy guarantee of $A$.
\label{Different Privacy} 
\end{mydef}

A general method for computing an approximation to any function $f$ (on $D$) while preserving $\epsilon$-differential privacy is the \textit{Laplace mechanism} \cite{dwork2006calibrating}, where the output of $f$ is a vector of real numbers. In particular, the mechanism exploits the global sensitivity of $f$ over any two neighboring databases (differing at most one record), which is denoted as $GS_{f}(D)$. Given $GS_{f}(D)$, the Laplace mechanism ensures $\epsilon$-differential privacy by injecting noise $\eta$ into each value in the output of $f(D)$ as
\begin{equation}
pdf(\eta) = \frac{\epsilon}{2GS_{f}(D)}exp(-|\eta|\cdot \frac{\epsilon}{GS_{f}(D)})
\end{equation}
where $\eta$ is drawn i.i.d. from Laplace distribution with zero mean and scale $GS_{f}(D)/\epsilon$.

Research in differential privacy has been significantly studied, from both the theoretical perspective, e.g.,
\cite{nipsChaudhuriM08,DBLP:conf/sigmod/KiferM11}, and the application perspective, e.g., data collection \cite{RAPPOR}, data streams \cite{Chan:2012}, stochastic gradient descents \cite{SongCS13}, recommendation \cite{mcsherry2009differentially}, regression \cite{nipsChaudhuriM08}, online learning \cite{JainKT12}, publishing contingency tables \cite{xiao2010differential}, and spectral graph analysis \cite{DBLP:conf/pakdd/WangWW13}.
The mechanisms of achieving differential privacy mainly include the classic approach of adding Laplacian noise \cite{dwork2006calibrating}, the exponential mechanism \cite{McSherry:2007}, and the functional perturbation approach \cite{nipsChaudhuriM08}. 

\subsection{Functional Mechanism Revisited} 
Functional mechanism \cite{zhang2012functional} is an extension of the Laplace mechanism.  It achieves $\epsilon$-differential privacy by perturbing
the objective function $f_D(W)$ and then releasing the model parameter $\overline{W}$ that minimizes the perturbed objective function $\overline{f}_D(W)$ instead of the original one.
%The classic Laplace mechanism (i.e., adding noise directly to the $\omega$) requires an accurate analysis on the sensitivity of $\omega$, which is very challenging given the complex relationship between $D$ and $\omega$ for a complicated  model $\rho$.
The functional mechanism exploits the polynomial representation of $f_D(W)$. The model parameter $W$ is a vector that contains $d$ values $W_1, \ldots, W_d$. Let $\phi(W)$ denote a product of $W_1, \ldots, W_d$, namely, $\phi(W) = W^{c_1}_1 \cdot W^{c_2}_2 \cdot \cdot \cdot W^{c_d}_d$ for some $c_1, \ldots, c_d \in \mathbb{N}$. Let $\Phi_j (j \in \mathbb{N})$ denote the set of all products of $W_1, \ldots, W_d$ with degree $j$, i.e., $\Phi_j = \big\{W^{c_1}_1 \cdot W^{c_2}_2 \cdot \cdot \cdot W^{c_d}_d \Big\vert \sum_{l = 1}^d c_l = j \big\}$. By the Stone-Weierstrass Theorem, any continuous and differentiable $f(t_i, W)$ can always be written as a polynomial of $W_1, \ldots, W_d$, for some $J \in [0, \infty]$, i.e., $f(t_i, W) = \sum_{j = 0}^J\sum_{\phi \in \Phi_j}\lambda_{\phi t_i}\phi(W)$ where $\lambda_{\phi t_i} \in \mathbb{R}$ denotes the coefficient of $\phi(W)$ in the polynomial. Note that $t_i$ and $\mathbf{x}_i$ are used exchangeably to indicate the data tuple $i$.

For instance, the polynomial expression of the loss function in the linear regression is as follows: $f(\mathbf{x}_i, W) = (y_i - \mathbf{x}_i^T W)^2 =  y_i^2 - \sum_{j = 1}^{d}(2y_ix_{ij})W_j + \sum_{1 \leq j,l \leq d}(x_{ij}x_{il})W_j W_l$. 
We can see that it only involves monomials in $\Phi_0 = \{1\}, \Phi_1 = \{W_1, \ldots, W_d\}$, and $\Phi_2 = \{W_i W_j \big\vert i,j \in [1, d]\}$. Each $\phi(W)$ has its own coefficient, e.g., for $W_j$, its polynomial coefficient $\lambda_{\phi_{t_i}} = -2y_ix_{ij}$. Similarly, $f_D(W)$ can also be expressed as a polynomial of $W_1, \ldots, W_d$. 
\begin{equation}
f_D(W) = \sum_{j = 0}^J\sum_{\phi \in \Phi_j}\sum_{t_i \in D}\lambda_{\phi t_i}\phi(W)  
\end{equation} 

\begin{lemma} \cite{zhang2012functional} Let $D$ and $D'$ be any two neighboring datasets. Let $f_D(W)$ and $f_{D'}(W)$ be the objective functions of regression analysis on $D$ and $D'$, respectively. The following inequality holds 
\begin{equation}
\Delta = \sum_{j = 1}^J \sum_{\phi \in \Phi_j}^2\Big\lVert \sum_{t_i \in D} \lambda_{\phi t_i} - \sum_{t'_i \in D'} \lambda_{\phi t'_i} \Big\rVert_1 \leq 2\max_t \sum_{j = 1}^J \sum_{\phi \in \Phi_j} \lVert \lambda_{\phi t} \rVert_1 \nonumber 
\end{equation}
where $t_i, t'_i$ or $t$ is an arbitrary tuple. 
\label{Lemma1}
\end{lemma}

\begin{wrapfigure}{l}{0.45\textwidth}
\vspace{-20pt}
  \begin{center}
    \includegraphics[width=2.3in]{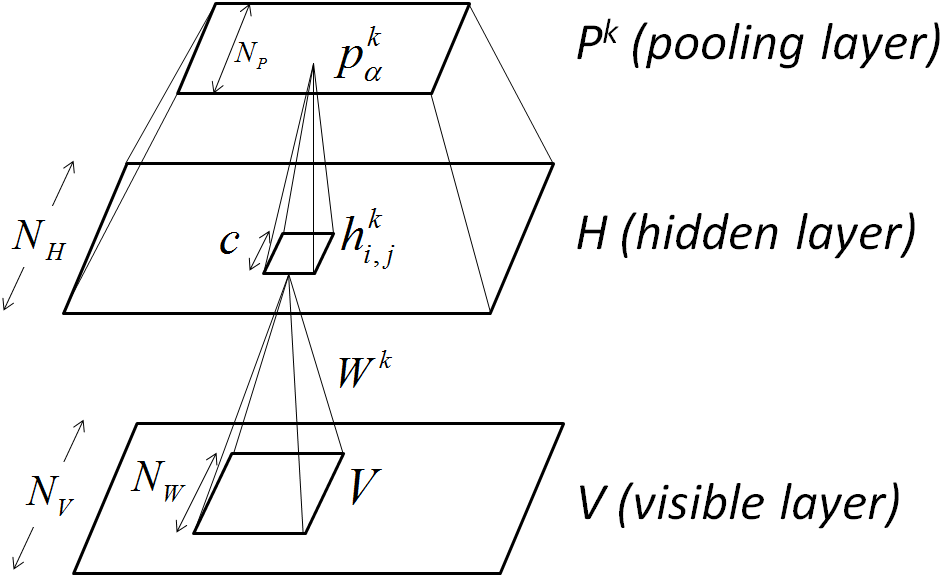}
    \caption{Convolutional Restricted Boltzmann Machine (CRBM).}
    \label{CNN}
\vspace{-40pt}
  \end{center}
\end{wrapfigure}

To achieve $\epsilon$-differential privacy, $f_D(W)$ is perturbed by injecting Laplace noise $Lap(\frac{\Delta}{\epsilon})$ into its polynomial coefficients $\lambda_{\phi}$, and then the model parameter $\overline{W}$ is derived to minimize the perturbed function $\overline{f}_D(W)$, where $\Delta = 2 \max_t \sum_{j = 1}^J \sum_{\phi \in \Phi_j} \lVert \lambda_{\phi t}\rVert_1$, according to the Lemma \ref{Lemma1}.

\subsection{Convolutional Deep Belief Networks}
The basic Convolutional Restricted Boltzmann Machine (CRBM) \cite{Lee:2009} consists of two layers: an input layer $V$ and a hidden layer $H$ (Figure \ref{CNN}). The layer of hidden units consists of $K$ groups, each of which is an $N_H \times N_H$ array of binary units. There are $N_H^2 K$ hidden units in total. Each group (in $K$ groups) is associated with a $N_W \times N_W$ filter, where $N_W = N_V - N_H + 1$. The filter weights are shared across all the hidden units within the group. In addition, each group of hidden units has a bias $b_k$, and all visible units share a single bias $c$. Training a CRBM is to minimize the following energy function $E(\mathsf{v}, \mathsf{h})$ as: 
\begin{multline}
E(\mathsf{v}, \mathsf{h}) = -\sum_{k=1}^K \sum_{i,j = 1}^{N_H} \sum_{r,s = 1}^{N_W} h_{ij}^k W_{rs}^k v_{i+r-1,j+s-1} - \sum_{k=1}^K b_k \sum_{i,j = 1}^{N_H} h_{ij}^k - c\sum_{i,j=1}^{N_V}v_{ij} 
\end{multline}

Gibbs sampling can be applied using the following conditional distributions:
\begin{align}
P(h_{ij}^k = 1| \mathsf{v}) = \sigma\big((\tilde{W}^k * v)_{ij} + b_k\big) \label{GibbH} \\
P(v_{ij} = 1| \mathsf{h}) = \sigma\big((\sum_k W^k * h^k)_{ij} + c\big)
\label{GibbV}
\end{align}
where $\sigma$ is the sigmoid function.

The energy function given the dataset $D$ is as follows: 
\begin{multline}
E(D, W) = -\sum_{t \in D} \sum_{k=1}^K \sum_{i,j = 1}^{N_H} \sum_{r,s = 1}^{N_W} h_{ij}^{k,t} W_{rs}^k v^t_{i+r-1,j+s-1} \\ 
- \sum_{t \in D}\sum_{k=1}^K b_k \sum_{i,j = 1}^{N_H} h_{ij}^{k,t} - c \sum_{t \in D} \sum_{i,j=1}^{N_V}v^t_{ij} 
\label{E(D, W)}
\end{multline}

The max-pooling layer plays the role of a signal filter. By stacking multiple CRBMs on top of each other, we can construct a \textit{convolutional deep belief network} (CDBN) \cite{Lee:2009}. Regarding the softmax layer, we use the cross-entropy error function for a binomial prediction task. Let $Y_T$ be a set of labeled data points used to train the model, the cross-entropy error function is given by 
\begin{equation}
C(Y_T, \theta) = -\sum_{i = 1}^{|Y_T|} \Big( y_i\log \hat{y}_i + (1-y_i)\log (1 - \hat{y}_i) \Big)
\label{cross-entropy error}
\end{equation}
where `$T$' in $Y_T$ is used to denote ``\textbf{t}raining" data. 

We can use the layer-wise unsupervised training algorithm \cite{Bengio07greedylayerwise} and \textit{back-propagation} to train CDBNs.

\subsection{Chebyshev Polynomials}
In principle, many polynomial approximation techniques, e.g., Taylor Expansion, Bernoulli polynomial, Euler polynomial, Fourier series, Discrete Fourier transform, Legendre polynomial, Hermite polynomial, Gegenbauer polynomial, Laguerre polynomial, Jacobi polynomial, and even the state-of-the-art techniques in the 20th century, including spectral methods and Finite Element methods \cite{Harper2012}, can be applied to approximate non-linear energy functions used in CDBNs. However, figuring out an appropriate way to use each of them is non-trivial. First, estimating the lower and upper bounds of the approximation error incurred by applying a particular polynomial in deep neural networks is not straightforward; it is very challenging. It is significant to have a strong guarantee in terms of approximation errors incurred by the use of any approximation approach to ensure model utility in deep neural networks. %In addition, estimating the global sensitivity for an approximation technique under differential privacy requires a significant expenditure of time and effort.
In addition, the approximation error bounds must be independent of the number of data instances to guarantee the ability to be applied in large datasets without consuming excessive privacy budgets.

With these challenging issues, Chebyshev polynomial really stands out. The most important reason behind the usage of Chebyshev polynomial is that the upper and lower bounds of the error incurred by approximating activation functions and energy functions can be estimated and proved, as shown in the next section. Furthermore, these error bounds do not depend on the number of data instances, as we will present in Section 3.4. This is a substantial result when working with complex models, such as deep neural networks on large-scale datasets. In addition, Chebyshev polynomials are well-known, efficient, and widely used in many real-world applications \cite{mason2002chebyshev}. Therefore, we propose to use Chebyshev polynomials in our work to preserve differential privacy in deep convolution belief networks. 

The four kinds of Chebyshev polynomials can be generated from the two-term recursion formula:
\begin{equation}
\text{\ \ \ \ \ \ } T_{k+1}(x) = 2xT_k(x) - T_{k-1}(x), \textit{\ \ \ \ \ \ } T_0(x) = 1
\end{equation}
with different choices of initial values $T_1(x) = x, 2x, 2x-1, 2x+1$.

According to the well-known result \cite{Chebyshev}, if a function $f(x)$ is the Riemann integrable on interval $[-1, 1]$, $f(x)$ can be presented in a Chebyshev polynomial approximation as follows: 
\begin{equation}
f(x) = \sum_{k=0}^{\infty} A_k T_k(x) = A' X\big(T(x)\big) 
\label{Chebyshev}
\end{equation}
where $A_k = \frac{2}{\pi} \int_{-1}^{1} \frac{f(x)T_k(x)}{\sqrt{1 - x^2}}dx$, $k \in \mathbb{N}$, $A' = [\frac{1}{2}A_0, \ldots, A_k, \ldots]$, $T_k(x)$ is Chebyshev polynomial of degree $k$, $X\big(T(x)\big) = [T_0(x) \ldots T_k(x) \ldots]$.

The closed form expression for Chebyshev polynomials of any order is: 
\begin{equation}
T_i(x) = \sum_{j = 0}^{[i/2]} (-1)^j \begin{pmatrix} i \\ 2j \end{pmatrix} x^{i - 2j}(1 - x^2)^j
\label{Chebyshev2} 
\end{equation}
where $[i/2]$ is the integer part of $\frac{i}{2}$.

\section{Private Convolutional Deep Belief Network}
In this section, we formally present our framework (Alg. \ref{Pseudo pCNN}) to develop a convolutional deep belief network under $\epsilon$-differential privacy. Intuitively, the algorithm used to develop dPAs can be applied to CDBNs. However, the main issue is that their approximation technique has been especially designed for cross-entropy error-based objective functions \cite{Bengio2009}. There are many challenging issues in adapting their technique in CDBNs. The cross entropy error-based objective function is very different from the energy-based objective function (Eq. \ref{E(D, W)}). As such: \textbf{(1)} It is difficult to derive its global sensitivity used in the functional mechanism, and \textbf{(2)} It is difficult to identify the approximation error bounds in CDBNs. To achieve private convolutional deep belief networks (pCDBNs), we figure out a new approach of using the Chebyshev Expansion \cite{Chebyshev} to derive polynomial approximations of non-linear energy-based objective functions (Eq. \ref{E(D, W)}), such that differential privacy can be preserved by leveraging the functional mechanism.

Our framework to construct the pCDBN includes four steps (Alg. \ref{Pseudo pCNN}). 
\begin{itemize}
\item First, we derive a polynomial approximation of energy-based function $E(D, W)$ (Eq. \ref{E(D, W)}), using the Chebyshev Expansion. The polynomial approximation is denoted as $\widehat{E}(D, W)$. 

\item Second, the functional mechanism is used to perturb the approximation function $\widehat{E}(D, W)$; the perturbed function is denoted as $\overline{E}(D, W)$. We introduce a new result of \textit{sensitivity} computation for CDBNs. Next, we train the model to obtain the optimal perturbed parameters $\overline{W}$ by using gradient descent. That results in private hidden layers, which are used to produce max-pooling layers. Note that we do not need to enforce differential privacy in max-pooling layers. This is because max-pooling layers play roles as signal filters only. 

\item Third, we stack multiple pairs of a private hidden layer and a max-pooling layer $(H, P)$ on top of each other to construct the private convolutional deep belief network (pCDBN). 

\item Finally, we apply the technique presented in \cite{Phan0WD16} to enforce differential privacy in the softmax layer for prediction and classification tasks.
\end{itemize}

Let us first derive the polynomial approximation form of $E(D,W)$ by applying Chebyshev Expansion, as follows.

\begin{algorithm}[t] 
\SetAlgoNoEnd
\SetKwInOut{Input}{Input}
\SetKwInOut{Output}{Output}
\LinesNumbered
\small{
1: Derive a polynomial approximation of the energy function $E(D, W)$ (Eq. \ref{E(D, W)}), denoted as $\widehat{E}(D, W)$ \\
2: The function $\widehat{E}(D, W)$ is perturbed by using \textit{functional mechanism} (FM) \cite{zhang2012functional}, the perturbed function is denoted as $\overline{E}(D, W)$ \\
3: Stack the private hidden and pooling layers \\
4: By using the technique in \cite{Phan0WD16}, we derive and perturb the polynomial approximation of the softmax layer $\widehat{C}(\theta)$ (Eq. \ref{Softmax}), the perturbed function is denoted as $\overline{C}(\theta)$, Return $\overline{\theta} = \arg \min_\theta \overline{C}(\theta)$\\
}
\caption{\textbf{Private Convolutional Deep Belief Network}}
\label{Pseudo pCNN}
\end{algorithm}

\subsection{Polynomial Approximation of the Energy Function}
There are two challenges in the energy function $E(D, W)$ that prevent us from applying it for private data reconstruction analysis: \textbf{(1)} Gibbs sampling is used to estimate the value of every $h^k_{ij}$; and \textbf{(2)} The probability of every $h^k_{ij}$ equal to 1 is a sigmoid function which is not a polynomial function with parameters $W^k$. Therefore, it is difficult to derive the sensitivity and error bounds of the approximation polynomial representation of the energy function $E(D, W)$. Perturbing Gibbs sampling is challenging. Meanwhile, injecting noise in the results of Gibbs sampling will significantly affect the properties of hidden variables, i.e., values of hidden variables might be out of their original bounds, i.e., $[0, 1]$. 

To address this, we propose to preserve differential privacy in the model before applying Gibbs sampling. The generality is still guaranteed since Gibbs sampling is applied for all hidden units. In addition, we need to derive an effective polynomial approximation of the energy function, so that differential privacy preserving is feasible. First, we propose to consider the probability $P(h_{ij}^k = 1| \mathsf{v}) = \sigma\big((W^k * v)_{ij} + b_k\big)$ instead of $h^k_{ij}$ in the energy function $E(D, W)$. The main goal of minimizing the energy function, i.e., ``\textit{the better the reconstruction of $v$ is, the better the parameters $W$ are,}" remains the same. Therefore, the generality of our proposed approach is still guaranteed. The energy function can be rewritten as follows:
\begin{multline}
\widetilde{E}(D, W) = \sum_{t \in D} \Big[ - \sum_{k=1}^K \sum_{i,j = 1}^{N_H} \sum_{r,s = 1}^{N_W} \sigma\big((W^k * v^t)_{ij} + b_k\big) \times W_{rs}^k v^t_{i+r-1,j+s-1} \\ 
- \sum_{k=1}^K b_k \sum_{i,j = 1}^{N_H} \sigma\big((W^k * v^t)_{ij} + b_k\big)
- c \sum_{i,j=1}^{N_V}v^t_{ij} \Big]
\end{multline}

As the sigmoid function $\sigma(\cdot)$ in neural networks satisfies the Reimann integrable condition \cite{Miroslav:2012}, it can be approximated by the Chebyshev series. We propose to derive a Chebyshev polynomial approximation function for the $\sigma\big((W^k * v^t)_{ij} + b_k\big)$ that results in a polynomial approximation function for our energy function $\widetilde{E}(\cdot)$. To make our sigmoid function satisfy the Riemann integrable condition on $[-1, 1]$, we rewrite it as: $\sigma\big(\frac{(W^k * v^t)_{ij} + b_k}{Z^k_{ij}}\big)$ where $Z^k_{ij}$ is a local response normalization (LRN) term which can be computed as follows: $Z^k_{ij} = \max\Big(\big|(W^k * v^t)_{ij} + b_k\big|, \Big[q + \alpha \sum_{m = \max (0, k-l/2)}^{\min (K-1, k + l/2)} \big((W^m * v^t)_{ij} + b_m\big)^2 \Big]^\beta \Big)$, where the constants $q, l, \alpha,$ and $\beta$ are hyper-parameters, $K$ is the total number of feature maps. As in \cite{krizhevsky2012imagenet}, we used $q = 2, l = 5, \alpha = 10^{-4},$ and $\beta = 0.75$ in our experiments.

%$Z_k = max_{i,j,t}\big|(W^k * v^t)_{ij} + b_k\big|$.

From Eq. \ref{Chebyshev}, the Chebyshev polynomial approximation of our sigmoid function is as follows:
\begin{equation}
\sigma\big(\frac{(W^k * v^t)_{ij} + b_k}{Z^k_{ij}}\big) = \sum_{l = 0}^{\infty} A_lT_l(\frac{(W^k * v^t)_{ij} + b_k}{Z^k_{ij}})
\label{Chebyshev sigmoid}
\end{equation}
where $A_l$ and $T_l$ can be computed using Eqs. \ref{Chebyshev} and \ref{Chebyshev2}.

Now, there is still a challenge that prevents us from applying the functional mechanism to preserve differential privacy in applying Eq. \ref{Chebyshev sigmoid}: The equation involves an infinite summation. To address this problem, we remove all orders greater than $L$. Based on the Chebyshev series, the polynomial approximation of the energy function $\widetilde{E}(\cdot)$ can be written as:
\begin{multline}
\widehat{E}(D, W) = \sum_{t \in D} \Big[ -\sum_{k=1}^K \sum_{i,j = 1}^{N_H} \sum_{r,s = 1}^{N_W} \Big(\sum_{l = 0}^{L} A_lT_l(\frac{(W^k * v^t)_{ij} + b_k}{Z^k_{ij}}) \Big) \times W_{rs}^k v^t_{i+r-1,j+s-1} \\ 
- \sum_{k=1}^K b_k \sum_{i,j = 1}^{N_H} \sum_{l = 0}^{L} A_lT_l(\frac{(W^k * v^t)_{ij} + b_k}{Z^k_{ij}}) - c \sum_{i,j=1}^{N_V}v^t_{ij} \Big]
\end{multline}

$\widehat{E}(\cdot)$ is a polynomial approximation function of the original energy function $E(\cdot)$ in Eq. \ref{E(D, W)}. Furthermore, the term $\sum_{l = 0}^{L} A_lT_l(\frac{(W^k * v^t)_{ij} + b_k}{Z^k_{ij}})$ can be rewritten as: $\sum_{l = 0}^L \alpha_l (\frac{(W^k * v^t)_{ij} + b_k}{Z^k_{ij}})^l$,
where $\alpha$ are the Chebyshev polynomial coefficients. For instance, given $L = 7$, we have $\sum_{l = 0}^{L = 7} A_lT_l(X) = \frac{1}{2^5}(-5X^7 + 21X^5 - 35X^3 + 35X + 16)$, where $X = \frac{(W^k * v^t)_{ij} + b_k}{Z^k_{ij}}$.

%We can rewrite $\widehat{E}(\cdot)$ as follows:

%$
%\widehat{E}(D, W) = \sum_{t \in D} \Big[ -\sum_{k=1}^K \sum_{i,j = %1}^{N_H} \sum_{r,s = 1}^{N_W} \\ \Big(\sum_{l = 0}^{L} \alpha_l %(\frac{(W^k * v^t)_{ij} + b_k}{Z_k})^l \Big) \times W_{rs}^k %v^t_{i+r-1,j+s-1} \\ - \sum_{k=1}^K b_k \sum_{i,j = 1}^{N_H} \sum_{l = %0}^{L} \alpha_l (\frac{(W^k * v^t)_{ij} + b_k}{Z_k})^l \\ - c \sum_{t \in D} %\sum_{i,j=1}^{N_V}v^t_{ij} \Big].
%$

\subsection{Perturbation of Objective Functions}
We employ the functional mechanism \cite{zhang2012functional} to perturb the objective function $\widehat{E}(\cdot)$ by injecting Laplace noise into its polynomial coefficients. The hidden layer contains $K$ groups of hidden units. Each group is trained with a local region of input neurons, which will not be merged with each other in the learning process. Therefore, it is not necessary to aggregate sensitivities of the training algorithm in $K$ groups to the sensitivity of the function $\widehat{E}(\cdot)$. Instead, the sensitivity of the function $\widehat{E}(\cdot)$ can be considered the maximal sensitivity given any single group. As a result, the sensitivity of the function $\widehat{E}(\cdot)$ can be computed in the following lemma.

\begin{lemma} Let $D$ and $D'$ be any two neighboring datasets. Let $\widehat E(D, W)$ and $\widehat E(D',W)$ be the objective functions of regression analysis on $D$ and $D'$, respectively. $\alpha$ are Chebyshev polynomial coefficients. The following inequality holds:
\begin{multline}
\Delta \leq 2\max_{t,k} \sum_{i,j =1}^{N_H} \sum_{l = 0}^L |\alpha_l| \Big[ (\frac{\sum_{r,s = 1}^{N_W} v^{t,k}_{ij,rs} + 1}{Z^k_{ij}})^l +
\\ \sum_{r,s = 1}^{N_W}(\frac{\sum_{r',s' = 1}^{N_W} v^{t,k}_{ij,r's'} + 1}{Z^k_{ij}})^l |v^{t,k}_{ij,rs}| \Big] + \sum_{i,j = 1}^{N_V}|v^t_{ij}|
\label{GlobalSensitivity}
\end{multline}
\label{LemmaGlobalSensitivity} 
\end{lemma}
\begin{proof}
By replacing $W^t_{rs}$ (i.e., $\forall r, s$), $b_k$, and $c$ in $\widehat{E}(D, W)$ with $1$, we have the function with only polynomial coefficients of $\widehat{E}(D, W)$, denoted $\lambda_{\phi D}$. We have that
\begin{equation}
\lambda_{\phi D} = \sum_{t \in D} \lambda_{\phi t}\nonumber
\end{equation}
where
\begin{multline}
\lambda_{\phi t} = -\sum_{k=1}^K \sum_{i,j = 1}^{N_H} \sum_{r,s = 1}^{N_W} \big(\sum_{l = 0}^{L} \alpha_l (\frac{\sum^{N_W}_{r',s'=1} v^t_{ij,r's'} + 1}{Z^k_{ij}})^l \big) v^t_{ij,rs} \\ - \sum_{k=1}^K b_k \sum_{i,j = 1}^{N_H} \sum_{l = 0}^{L} \alpha_l (\frac{\sum^{N_W}_{r,s=1} v^t_{ij,rs} + 1}{Z^k_{ij}})^l - \sum_{i,j=1}^{N_V}v^t_{ij} \nonumber
\end{multline}

The $\widehat{E}(\cdot)$'s sensitivity can be computed as follows:
\begin{multline}
\Delta = \Big\lVert \sum_{t_i \in D} \lambda_{\phi t_i} - \sum_{t'_i \in D'} \lambda_{\phi t'_i} \Big\rVert_1 \leq 2\max_t \lVert \lambda_{\phi t} \rVert_1 \\
\leq 2\max_t \sum_{k = 1}^{K} \sum_{i,j =1}^{N_H} \sum_{l = 0}^L |\alpha_l| \Big[ (\frac{\sum_{r,s = 1}^{N_W} v^{t,k}_{ij,rs} + 1}{Z^k_{ij}})^l +
\\  \sum_{r,s = 1}^{N_W}(\frac{\sum_{r',s' = 1}^{N_W} v^{t,k}_{ij,r's'} + 1}{Z^k_{ij}})^l |v^{t,k}_{ij,rs}| \Big] + \sum_{i,j = 1}^{N_V}|v^t_{ij}|
\label{semiGlobalSensitivity}
\end{multline}

The current sensitivity is an aggregation of sensitivities from all $K$ groups of hidden units. However, each of them is trained with a local region of input neurons, which will not be merged with the others in the learning process. Therefore, the sensitivity of the function $\widehat{E}(\cdot)$ can be considered the maximal sensitivity given any single group of hidden units in a hidden layer. From Eq. \ref{semiGlobalSensitivity}, the final sensitivity of the function $\widehat{E}(\cdot)$ is as follows:
\begin{multline}
\Delta \leq 2\max_{t,k} \sum_{i,j =1}^{N_H} \sum_{l = 0}^L |\alpha_l| \Big[ (\frac{\sum_{r,s = 1}^{N_W} v^{t,k}_{ij,rs} + 1}{Z^k_{ij}})^l +
\\ \sum_{r,s = 1}^{N_W}(\frac{\sum_{r',s' = 1}^{N_W} v^{t,k}_{ij,r's'} + 1}{Z^k_{ij}})^l |v^{t,k}_{ij,rs}| \Big] + \sum_{i,j = 1}^{N_V}|v^t_{ij}| \nonumber
\end{multline}
Consequently, the Eq. \ref{GlobalSensitivity} holds.
\end{proof}

We use gradient descent to train the perturbed model $\overline{E}(\cdot)$. That results in \textit{private hidden layers}. To construct a private convolutional deep belief network (pCDBN), we stack multiple private hidden layers and max-pooling layers on top of each other. The pooling layers only play the roles of signal filters of the private hidden layers. Therefore, there is no need to enforce privacy in max-pooling layers. 

\subsection{Perturbation of Softmax Layer}
On top of the pCDBN, we add an output layer, which includes a single binomial variable to predict $Y$. The output variable $\hat{y}$ is fully linked to the hidden variables of the highest hidden (pooling) layer, denoted $p_{(o)}$, by weighted connections $W_{(o)}$, where $o$ is the number of hidden (pooling) layers in the CDBNs. We use the logistic function as an activation function of $\hat{y}$, i.e., $\hat{y} = \sigma(W_{(o)} p_{(o)})$. \textit{Cross-entropy error}, which is a typical objective function in deep learning \cite{Bengio2009}, is used as a loss function. The cross-entropy error function has been widely used and applied in real-world applications \cite{Bengio2009}. Therefore, it is critical to preserve differential privacy under the use of the cross-entropy error function. However, other loss functions, e.g., square errors, can be applied in the softmax layer as well. Let $Y_T$ be a set of labeled data points used to train the model, the cross-entropy error function is given by: 
\begin{equation}
C(Y_T, \theta) = -\sum_{i = 1}^{|Y_T|} \Big( y_i\log (1+e^{-W_{(o)} p_{i(o)}})
+ (1-y_i)\log (1+e^{W_{(o)} p_{i(o)}}) \Big)
\label{Softmax} 
\end{equation}

By applying the technique in \cite{Phan0WD16}, based on Taylor Expansion \cite{tagkey1985}, we can derive the polynomial approximation of the cross-entropy error function as follows: 
\begin{equation}
\widehat{C}(Y_T, \theta) = \sum_{i = 1}^{|Y_T|} \sum_{l=1}^{2} \sum_{R = 0}^{2} \frac{f^{(R)}_{l}(0)}{R!}\big(W_{(o)}p_{i(o)}\big)^R
\label{PolyCrossEntropy}
\end{equation}
where 
\begin{equation}
\begin{split}
g_{1}(t_i, W_{(o)}) = W_{(o)}p_{i(o)} & \textit{\ \ \ ,\ \ \ } g_{2}(t_i, W_{(o)}) = W_{(o)}p_{i(o)} \\
f_{1}(z) = y_{i}\log (1 + e^{-z}) & \textit{\ \ \ ,\ \ \ } f_{2}(z) = (1-y_{i})\log (1 + e^{z}) \nonumber
\end{split}
\end{equation}

To preserve the differential privacy, the softmax layer is perturbed by using the \textit{functional mechanism} \cite{Phan0WD16,zhang2012functional}. The sensitivity of the softmax layer, $\Delta_C$, is estimated as $\Delta_C = |p_{(o)}| + \frac{1}{4}| p_{(o)}|^2$ \cite{Phan0WD16}.

\subsection{Approximation Error Bounds}
The following lemma shows how much error our approximation approaches incur. The average error of the approximations is always bounded, as presented in the following lemma:

\begin{lemma}{Approximation Error bounds.} Let $S_L(E) = \lVert E(D, W) - \widehat{E}(D, W) \rVert$, $U_L(E) = \lVert E(D, W) - E^*(D, W)\rVert$ where $E(D, W)$ is the target function, $\widehat{E}(D, W)$ is the approximation function learned by our model, and $E^*(D, W)$ is the best uniform approximation. The lower and upper bounds of the sum square error are as follows:
\begin{equation}
\Big(4 + \frac{4}{\pi^2}\log L\Big)N_H^2 K \times U_L(E) > S_L(E) \geq U_L(E) \geq \frac{\pi}{4} N_H^2 K \vert A_{L + 1}\vert
\label{EBounds}
\end{equation}
\label{Upper bounds} 
\end{lemma}
\begin{proof}
As the well-known results in \cite{Chebyshev}, given a target sigmoid function $\sigma$, a polynomial approximation function $\hat{\sigma}$ learned by the model, and the best uniform approximation of $\sigma$, $S_L(\sigma) = \lVert \sigma - \hat{\sigma} \rVert$, $U_L(\sigma) = \lVert \sigma - \sigma^*\rVert$, we have that: 
\begin{equation}
S_L(\sigma) \geq U_L(\sigma) \geq \frac{\pi}{4} \vert A_{L + 1}\vert 
\end{equation}

Since there are $N_H^2 K$ hidden units in our pCDBN model, we have $S_L(E) \geq U_L(E) \geq \frac{\pi}{4} N_H^2 K \vert A_{L + 1}\vert$. Similarly, in \cite{Chebyshev}, we also have
\begin{equation}
U_L(\sigma) \leq S_L(\sigma) < \Big(4 + \frac{4}{\pi^2}\log L\Big) U_L(\sigma)
\end{equation}

Since there are $N_H^2 K$ hidden units in our pCDBN model, we have 
\begin{equation}
\Big(4 + \frac{4}{\pi^2}\log L\Big)N_H^2 K \times U_L(E) > S_L(E) \geq U_L(E)
\end{equation} 

Therefore, the Eq. \ref{EBounds} holds.
\end{proof}

The approximation error depends on the structure of the energy function $E(D, W)$, i.e., the number of hidden neurons $N^2_H K$ and $\vert A_{L + 1}\vert$, and the number of attributes of the dataset. Lemma \ref{Upper bounds} can be used to determine when it should stop learning the approximation model. For each group of $N_H^2$ hidden units, the upper bound of the sum square error is only $\frac{\pi}{4}N_H^2 \vert A_{L + 1}\vert$, i.e., $\vert A_{L + 1}\vert$ is tiny when $L$ is large enough.

Importantly, Lemmas \ref{LemmaGlobalSensitivity} and \ref{Upper bounds} show that the sensitivity $\Delta$ and the approximation error bounds of the energy-based function are entirely independent of the number of data instances. This sufficiently guarantees that our differential privacy preserving framework can be applied in large datasets without consuming excessive privacy budgets. This is a substantial result when working with complex models, such as deep neural networks on large-scale datasets. It is worth noting that non-linear activation functions, which are continuously differentiable (Stone-Weierstrass Theorem \cite{WalterRudin}) and satisfy the Riemann-integrable condition, can be approximated by using Chebyshev Expansion. Therefore, our framework can be applied given such activation functions as, e.g., tanh, arctan, sigmoid, softsign, sinusoid, sinc, Gaussian, etc. \cite{Activation}. In the experiment section, we will show that our approach leads to accurate results.

Note that the proofs of Lemmas \ref{LemmaGlobalSensitivity} and \ref{Upper bounds} do not depend on the assumption of the data features being non-negative, and that the target follows by a binomial distribution. The proofs are generally applicable for inputs and the target, which are not restricted by any constraint. As shown in the next section, our approach efficiently works with a multi-class classification task on the MNIST dataset \cite{Lecun726791}. The cross-entropy error function is applied in the softmax layer.

%\textbf{Regularization.} %To avoid over-fitting and achieve bounded $\overline{RE}(D, W)$ and $\overline{C}(\theta)$, we add regularizations into the objective functions. As shown in previous sections, we transform the objective functions into quadratic polynomial forms and then inject noise into the coefficients of the polynomial forms to ensure privacy. Let $\hat{f}(D, W) = W^T M W + \alpha W + \beta$ be the matrix representation of the quadratic polynomial, and $\overline{f}(D, W) = W^T M^* W + \alpha^* W + \beta^*$ be the noisy version of $\widehat{f}(D, W)$ after injection of Laplace noise, where $\overline{f}$ could be either $\overline{RE}$ or $\overline{C}$. Then, $M$ must be symmetric and positive definite \cite{strang09}. To enforce regularization, we add a positive constant $\xi$, which is also called \textit{weight decay term}, to each entry in the main diagonal of $M^*$, such that the noisy objective function becomes
%\begin{equation}
%\overline{f}(D, W) = W^T (M^* + \xi I)W + \alpha^* W + \beta^*
%\end{equation}
%where $I$ is a $(d\times b) \times (d\times b)$ identity matrix if $\overline{f}$ is the $\overline{RE}$, otherwise $I$ is a $|\mathfrak{h}_{(k)}| \times |\mathfrak{h}_{(k)}|$ identity matrix if $\overline{f}$ is the $\overline{C}$. $\alpha^*$ and $\beta^*$ are the noisy versions of $\alpha$ and $\beta$, respectively.

\section{Experiments}
To validate our approach, we have conducted an extensive experiment on well-known and large-scale datasets, including a health social network, YesiWell data \cite{Phan0WD16}, and a handwriting digit dataset, MNIST \cite{Lecun726791}. Our task of validation focuses on four key issues: \textbf{(1)} The effectiveness and robustness of our pCDBN model; \textbf{(2)} The effects of our model and hyper-parameter selections, including the use of Chebyshev polynomial, the impact of the polynomial degree $L$, and the effect of probabilities $P(h^k_{ij}=1|v)$ in approximating the energy function; \textbf{(3)} The ability to work on large-scale datasets of our model; and \textbf{(4)} The benefits of being independent of the number of training epochs in consuming privacy budget.

We carry out the validation through three approaches. One is by conducting the human behavior prediction with various settings of data cardinality, privacy budget $\epsilon$, noisy vs. noiseless models, and original vs. approximated models. By this we rigorously examine the effectiveness of our model compared with the state-of-the-art algorithms, i.e., \cite{Phan0WD16,Abadi}. The second approach is to discover gold standards in our model configuration by examining various settings of hyper-parameters. The third approach is to access the benefits of being independent of the number of training epochs in terms of consuming privacy budget of our pCDBN model. In fact, we present the prediction accuracies of our pCDBN and existing algorithms as a function of the number of training epochs.

\subsection{Human Behavior Modeling} 
In this experiment, we have developed a \textit{private convolutional deep belief network} (pCDBN) for human behavior prediction and classification tasks in the YesiWell health social network \cite{Phan0WD16}.

\textbf{Health Social Network Data.} To be able to compare our model with the state-of-the-art deep private auto-encoders for human behavior prediction (dPAH), we use the same dataset used in Phan et al. \cite{Phan0WD16}. Data were collected from Oct 2010 to Aug 2011 as a collaboration between PeaceHealth Laboratories, SK Telecom Americas, and the University of Oregon to record daily physical activities, social activities (text messages, competitions, etc.), biomarkers, and biometric measures (cholesterol, BMI, etc.)  for a group of 254 overweight and obese individuals. Physical activities, including information about the number of walking and running steps, were reported via a mobile device carried by each user. All users enrolled in an online social network, allowing them to friend and communicate with each other. Users' biomarkers and biometric measures were recorded via daily/weekly/monthly medical tests performed at home individually or at our laboratories. 

In total, we consider three groups of attributes: 
\begin{itemize}
\item Behaviors: \#competitions joined, \#exercising days, \#goals set, \#goals achieved, $\sum$(distances), avg(speeds);
\item \#Inbox Messages: Encouragement, Fitness, Followup, Competition, Games, Personal, Study protocol, Progress report, Technique, Social network, Meetups, Goal, Wellness meter, Feedback, Heckling, Explanation, Invitation, Notice, Technical fitness, Physical;
\item Biomarkers and Biometric Measures: Wellness Score, BMI, BMI slope, Wellness Score slope.
\end{itemize}

\textbf{pCDBN for Human Behavior Modeling.} Our starting observation is that a human behavior is the outcome of behavior determinants such as \textit{self-motivation}, \textit{social influences}, and \textit{environmental events}. This observation is rooted in \textit{human agency in social cognitive theory}~\cite{Bandura89humanagency}. In our model, these human behavior determinants are combined together to model human behaviors. Given a tuple $t_i$, $x_{i1}, \ldots, x_{id}$ are the personal attributes and $y_i$ is a binomial parameter that indicates whether a user increases or decreases his/her exercises. To describe the pCDBN model, we adjust the notations $x_{i1}$ and $y_i$ a little bit to denote the temporal dimension, and our social network information. Specifically, $x^t_{u} = \{x_{1u}^t, \ldots, x_{du}^t\}$ is used to denote the $d$ attributes of user $u$ at time point $t$. Meanwhile, $y^t_{u}$ is used to denote the status of the user $u$ at time point $t$. $y^t_{u} = 1$ denotes $u$ increases exercises at time $t$; otherwise $y^t_{u} = 0$.  

In fact, the current behavior at time-stamp $t$ of a user $u$ is conditional on his/her behavior in the past $N$ time-stamps, i.e., $t-N, \ldots, t-1$. To model this effect (i.e., also considered as a form of self-motivation in social cognitive theory \cite{Bandura89humanagency}), we first aggregate his personal attributes in the last $N$ timestamps into a $d \times N$ matrix, which will be considered the visible input $V$. Then, to model self-motivation and social influence, we add an aggregation of his/her attributes and the effects from his/her friends at the current timestamp $t$ into the dynamic biases, i.e., $\hat{b}^t_{k}$ and $\hat{c}^t$, of the hidden and visible units (Eqs. \ref{HCDBN} - \ref{biasV}). The hidden and visible variables at time $t$ are
\begin{align}
& h^k_{i,j, t} = \sigma\big((\tilde{W}^k * v^t)_{ij} + \hat{b}^t_k\big) \label{HCDBN}\\
& v_{i,j, t} = \sigma\big((\sum_k W^k * h^k)^t_{ij} + \hat{c}^t\big) 
\end{align}
where 
\begin{align}
& \hat{b}^t_{k} = b_k + \sum_{e = 1}^d B^k_e  x^t_{eu} + \frac{\eta_k}{|F_u|}\sum_{v \in F_u} \psi_t(v, u) \label{biasH} \\
& \hat{c}^t = c + \sum_{e = 1}^d A_e  x^t_{eu} + \frac{\eta}{|F_u|}\sum_{v \in F_u} \psi_t(v, u) \label{biasV}
\end{align} 
where $\hat{b}^t_{k}$ and $\hat{c}^t$ are dynamic biases, $B^k_{e}$ is a matrix of weights which connects $x^t_{u}$ with hidden variables in the group $k$. $\psi_t(v, u)$ is the probability that $v$ influences $u$ on physical activity at time $t$. $\psi_t(v, u)$ is derived from the TaCPP model \cite{Phan7325206}. $F_u$ is a set of friends of $u$ in the social network. $\eta$ and $\eta_k$ are parameters which present the ability to observe the explicit social influences from neighboring users. 

The model includes two hidden layers. We trained 10 first layer bases, each $4 \times 12$ variables $v$, and 10 second layer bases, each $2 \times 6$. The pooling ratio was $2$ for both layers. In our work, contrastive divergent algorithm \cite{Hinton2002} was used to optimize the energy function, and back-propagation was used to optimize the cross-entropy error function in the softmax layer. The implementations of our models using Tensorflow\footnote{\url{https://www.tensorflow.org}} and Python were made publicly available on GitHub\footnote{\url{https://github.com/haiphanNJIT/PrivateDeepLearning}}. The results and algorithms can be reproduced on either a single workstation or a Hadoop cluster. To examine the effectiveness of our pCDBN, we established two experiments, i.e., prediction and classification, as follows.

\subsubsection{Human Behavior Prediction} 
\textbf{Experimental Setting.} Our pCDBN model is used to predict the statuses of all the users in the next time point $t+1$ given $\mathcal{M}$ past time points $t-\mathcal{M}+1,$...$, t$. The model has been trained on \textit{daily} and \textit{weekly} datasets. Both datasets contain 300 time points, 30 attributes, 254 users, 2,766 messages, 1,383 friend connections, 11,458 competitions, etc. For each dataset, we have, in total, $254$ users $\times 300$ timestamps $= 76,200$ data points. 

The number of previous time intervals $N$ is set to 4. $N$ is used as a time window to generate training samples. For instance, given 10 days of data ($\mathcal{M} = 10$), a time window of 4 days $N = 4$, and $d$ data features, e.g., BMI, \#steps, etc., a single input $V$ will be a $d \times N (= d \times 4)$ matrix. A single input $V$ is considered as a data sample to model human behavior in our prediction model. If we move the window $N$ on 10 days of data, i.e., $\mathcal{M}$, we will have $\mathcal{M} - N + 1$ training samples for each individual, i.e., $10 - 4 + 1 = 7$ in this example. So, we have, in total, $254 (\mathcal{M}$ - $N$ + $1)$ = $254 \times 7 = 1,778$ training samples for every 10 days of data $\mathcal{M}$ to predict whether an individual will increase physical activity in the next day $t + 1$. 

The Chebyshev polynomial approximation degree $L$ and learning rates are set to 7 and $10^{-3}$. To avoid over-fitting, we apply the $L1$-regularization and the dropout technique \cite{srivastava14a}, i.e., the dropout probability is set to 0.5. Regarding $\mathcal{K}$-fold cross-validation or bootstrapping, it is either unnecessary or impractical to apply them in deep learning, and particularly in our study \cite{BengioCrossValidation,ReedLASER14}. This is because: \textbf{(1)} It is too expensive and time consuming to train $\mathcal{K}$ deep neural networks, each of which usually has a large number of parameters, e.g., hundreds of thousands of parameters \cite{BengioCrossValidation}; and \textbf{(2)} %With state-of-the-art breakthroughs in deep learning, such as contrastive divergent (CD) algorithm \cite{Hinton2002}, there is no need to perform cross-validation in deep learning.
Bootstrapping is only used to train neural networks when class labels may be missing, objects in the image may not be localized, and in general, the labeling may be subjective, noisy, and incomplete \cite{ReedLASER14}. This is out of the scope of our focus. Our models were trained on a graphic card NVIDIA GTX TITAN X, 12 GB RAM with 3072 CUDA cores.

\textbf{Competitive Models.} We compare our pCDBN with two types of state-of-the-art models, as follows: 
\begin{itemize}
\item [\textbf{a)}] \textbf{Deep learning models for human behavior prediction}, including: (1) The original convolutional deep neural network (\textbf{CDBN}) for human behavior prediction without enforcing differential privacy; (2) The truncated version of the CDBNs, in which the energy function is approximated without injecting noise to preserve differential privacy, denoted \textbf{TCDBN}; and (3) The conditional Restricted Boltzmann Machine, denoted \textbf{SctRBM} \cite{KangLi2014}. None of these models enforces $\epsilon$-differential privacy. 

\item [\textbf{b)}] \textbf{Deep Private Auto-Encoder} (\textbf{dPAH}) \cite{Phan0WD16}, which is the state-of-the-art deep learning model under differential privacy for human behavior prediction. The dPAH model outperforms general methods for regression analysis under $\epsilon$-differential privacy, i.e., functional mechanism \cite{zhang2012functional}, DPME \cite{conf/nips/Lei11}, and filter-priority \cite{cormode2011personal}. Therefore, we only compare our model with the dPAH. 
\end{itemize}

%It is important to note that the implementation of the deep learning models under differential privacy proposed by Abadi et al. \cite{Abadi} and Shokri et al. \cite{ShokriVitaly2015} were designed for a different context, i.e., distributed manners, compared to our work. Thus, we will not include their models in our experiments.

\begin{figure}[t]
\centering
$\begin{array}{c@{\hspace{0.04in}}c@{\hspace{0.08in}}c}
\includegraphics[width=2.375in]{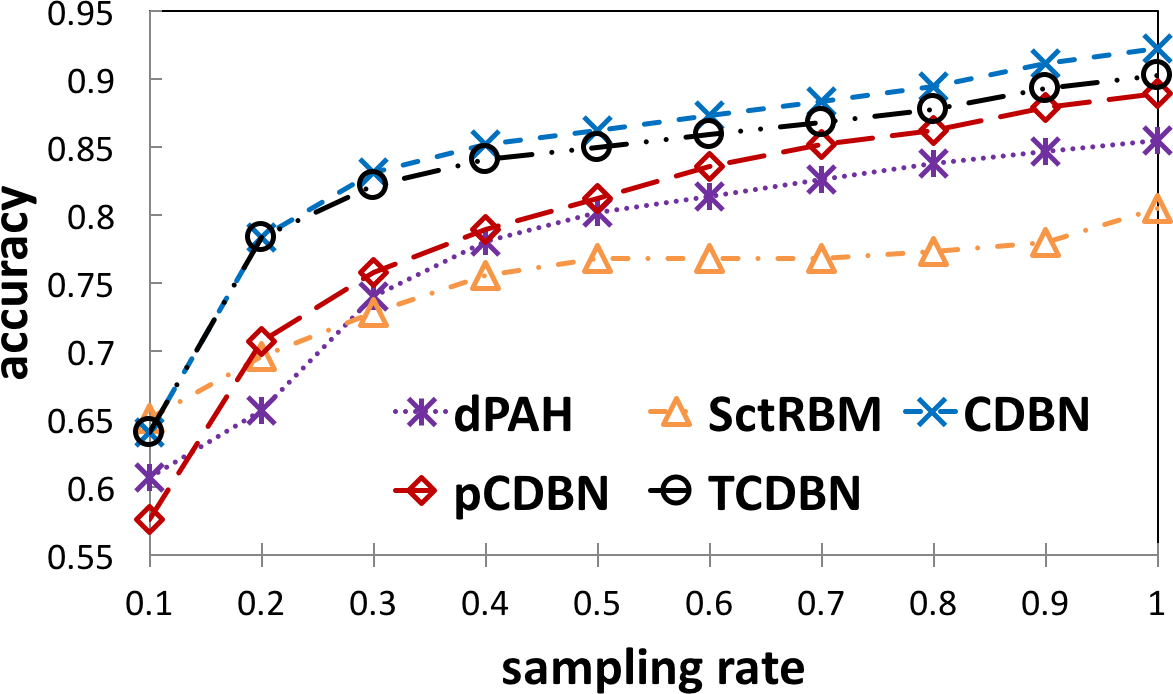} & \includegraphics[width=2.375in]{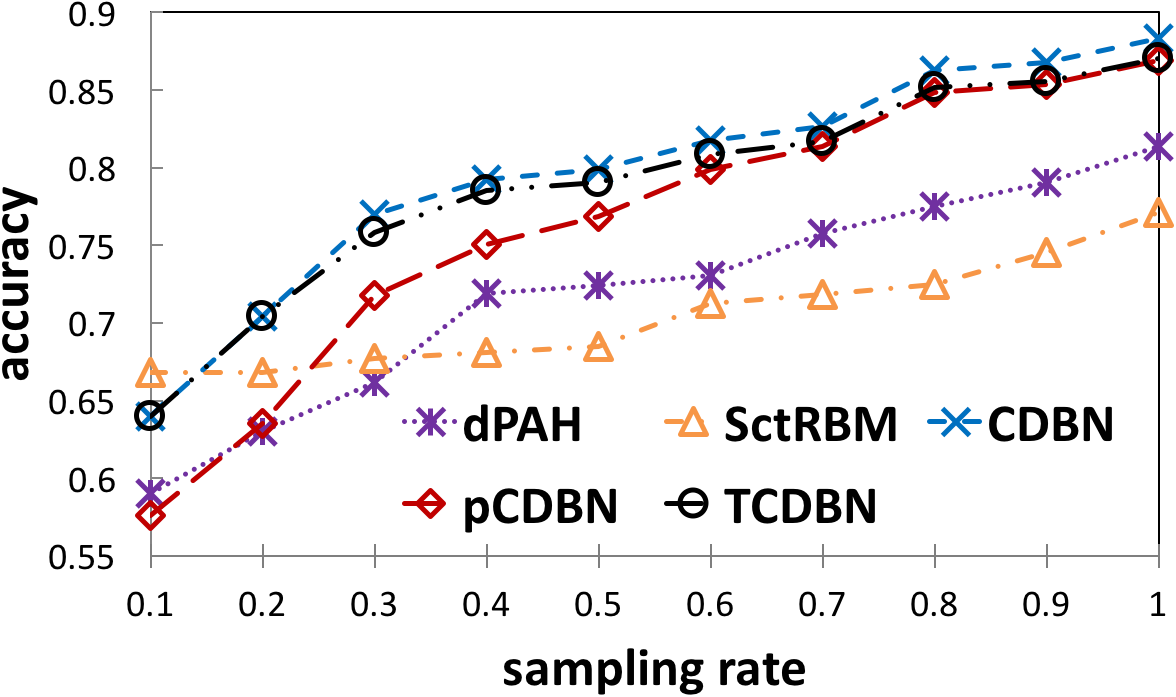} \\ [0.0cm] \mbox{(a) Weekly Dataset} & \mbox{(b) Daily Dataset}
\end{array}$
\caption{Prediction accuracy vs. dataset cardinality.}
\label{cardinality}
\end{figure}
\begin{figure}[!t]
\raggedleft
$\begin{array}{c@{\hspace{0.04in}}c@{\hspace{0.08in}}c}
\includegraphics[width=2.375in]{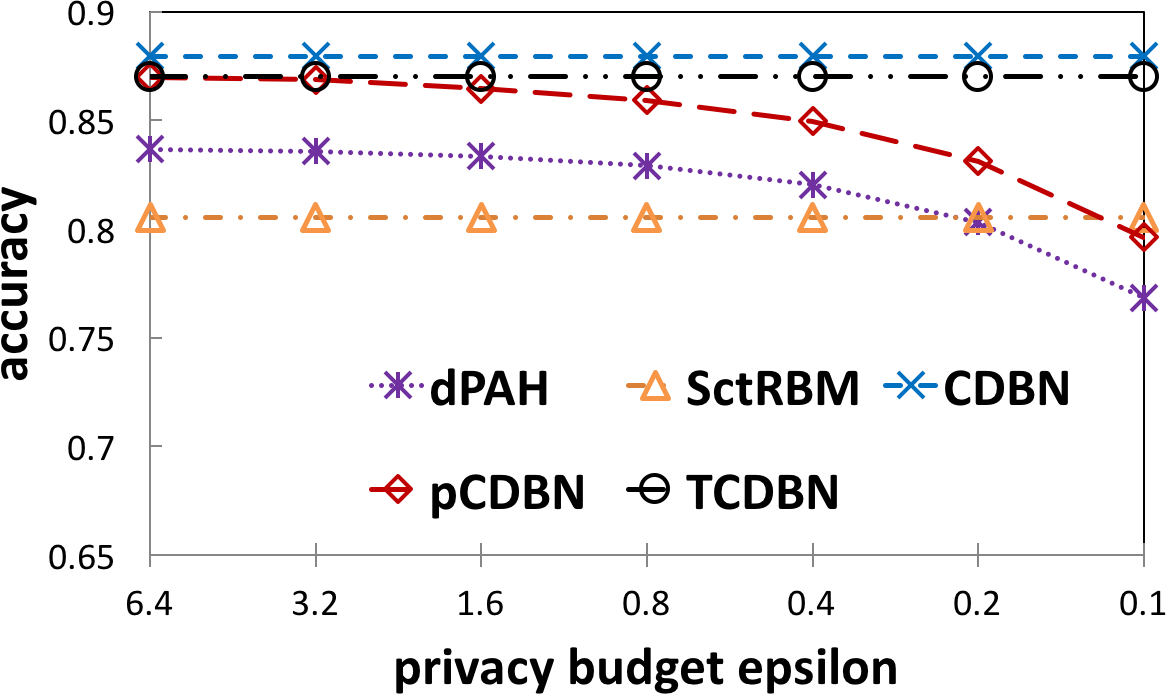} & \includegraphics[width=2.375in]{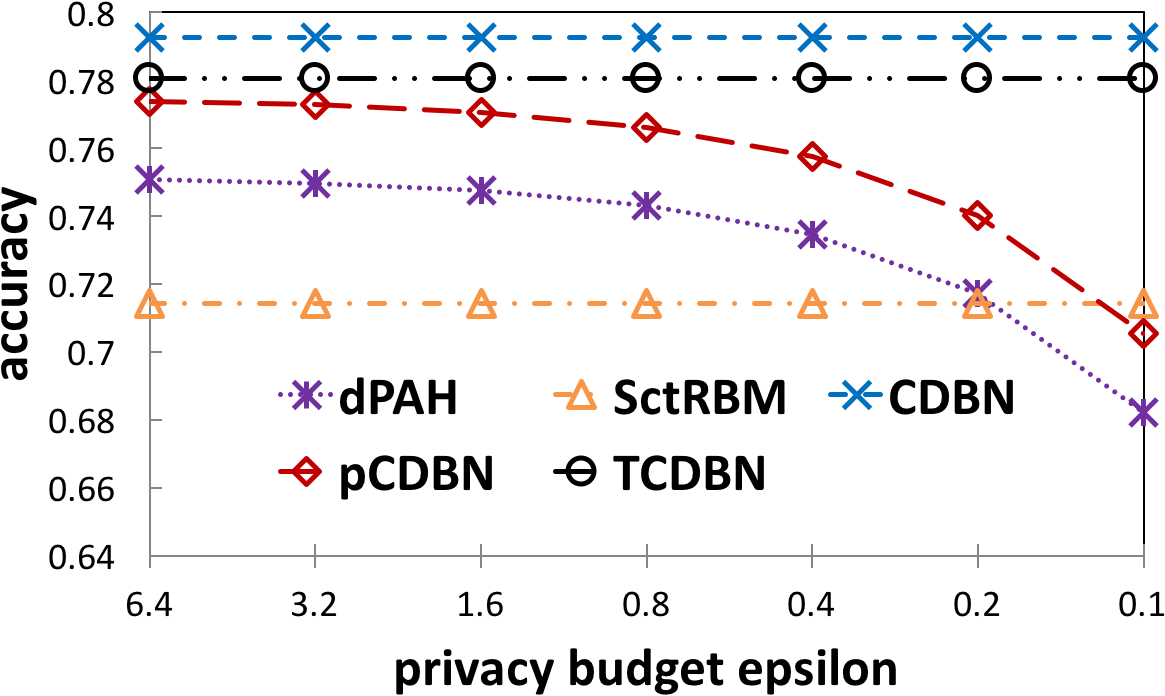} \\ [0.0cm] \mbox{(a) Weekly Dataset} & \mbox{(b) Daily Dataset}
\end{array}$
\caption{Prediction accuracy vs. privacy budget $\epsilon$.} 
\label{epsilon}
\end{figure}

$\bullet$ \textbf{Accuracy vs. Dataset Cardinality.} Fig. \ref{cardinality} shows the prediction accuracy of each algorithm as a function of the dataset cardinality. We vary the size of $\mathcal{M}$, which also can be considered as the sampling rate of the dataset. $\epsilon$ is 1.0 in this experiment. In both datasets, there is a gap between the prediction accuracy of pCDBN and the original convolutional deep belief network (CDBN). However, the gap dramatically gets smaller with the increase of the dataset cardinality ($\mathcal{M}$). In addition, our pCDBN outperforms the state-of-the-art dPAH in most of the cases, and the results are statistically significant ($p = 5.3828e$-$05$, performed by paired t-test). It also is significantly better than the SctRBM when the sampling rate goes just a bit higher, i.e., $> 0.2$ or $> 0.3$ ($p = 8.8350e$-$04$, performed by paired t-test). Either 0.2 or 0.3 is a small sampling rate; thus, this is a remarkable result.

$\bullet$ \textbf{Accuracy vs. Privacy Budget.} Fig. \ref{epsilon} illustrates the prediction accuracy of each model as a function of the privacy budget $\epsilon$. $\mathcal{M}$ is set to 12 $\approx 0.32\%$. The prediction accuracies of privacy non-enforcing models remain unchanged for all $\epsilon$. Since a smaller $\epsilon$ requires a larger amount of noise to be injected, privacy enforcing models incur higher inaccurate prediction results when $\epsilon$ decreases. pCDBN outperforms dPAH in all cases, and the results are statistically significant ($p = 2.7266e$-$12$, performed by paired t-test). In addition, it is relatively robust against the change of $\epsilon$. In fact, the pCDBN model is competitive even with privacy non-enforcing models, i.e., SctRBM.

$\bullet$ \textbf{Probabilities $P(h^k_{ij}=1|v)$ and Gibbs Sampling.} To approximate the energy function $E(D, W)$, we propose to use the probabilities $P(h^k_{ij}=1|v)$ instead of the values of $h^k_{ij}$, which are estimated by applying Gibbs Sampling on the $P(h^k_{ij}=1|v)$. To illustrate the effect of our approach, we conducted both theoretical analysis and experimental evaluations as follows. Let's use $h^k_{ij}$ to estimate the sensitivity $\Delta$ of the energy function $E(D,W)$ (Eq. \ref{E(D, W)}) by following Lemma \ref{Lemma1} as follows:
\begin{multline}
\Delta = 2\max_{t,k} \sum_{i,j =1}^{N_H} \sum_{r,s = 1}^{N_W} \big\lvert h_{ij}^{k,t} v^t_{i+r-1,j+s-1} \big\rvert + \sum_{i,j =1}^{N_H} \big\lvert h_{ij}^{k,t} \big\rvert + \sum_{i,j=1}^{N_V} \big\lvert v^t_{ij} \big\rvert
\label{GlobalSensitivity2}
\end{multline}

There are several issues in the Eq. \ref{GlobalSensitivity2} that prevent us from applying it. First, $h^k_{ij}$ cannot be considered an observed variable, since its value can only be estimated by applying Gibbs sampling from observed variables $v$ and parameters $W$. In other words, the value of $\Delta$ is significantly dependent on Gibbs sampling given $P(h^k_{ij}=1|v)$. Therefore, $\Delta$ can be uncertain in every sampling step. That may lead to a violation of the guarantee of privacy protection under a differential privacy mechanism. To address this issue, one may set all the hidden variables $h^k_{ij}$ to 1. That leads to the use of a maximal value of the sensitivity $\Delta$ as follows:
\begin{equation}
\Delta = 2\max_{t,k} \sum_{i,j =1}^{N_H} \sum_{r,s = 1}^{N_W} \big\lvert v^t_{i+r-1,j+s-1} \big\rvert + N_H^2 + \sum_{i,j=1}^{N_V} \big\lvert v^t_{ij} \big\rvert
\label{MaximalDelta}
\end{equation}

The maximal value of $\Delta$ (Eq. \ref{MaximalDelta}) is huge and is not an optimal bound. In other words, the model efficiency will be affected, since too much noise will be unnecessarily injected into the model.

To tackle this challenge, our solution is to consider the probabilities $P(h^k_{ij}=1|v)$ instead of $h^k_{ij}$. As a result, the sensitivity $\Delta$ in Lemma \ref{LemmaGlobalSensitivity} is only dependent on observed variables $v$ instead of Gibbs samplings. That leads to a smaller amount of noise injected into the model. To demonstrate the effect of this approach, our model is compared with its truncated version, in which the energy function is approximated without injecting noise to preserve differential privacy, denoted \textbf{TCDBN}. Experimental results illustrated in Figs. \ref{cardinality} and \ref{epsilon} show that the impact of our approach on the original model \textbf{CDBN} is marginal in terms of prediction accuracy. On average, the prediction accuracy is only less than 1\% lower compared with the original model. This is a practical result.

\begin{figure*}[t]
\raggedleft
$\begin{array}{c@{\hspace{0.0in}}c@{\hspace{0.0in}}c}
\includegraphics[width=2.45in]{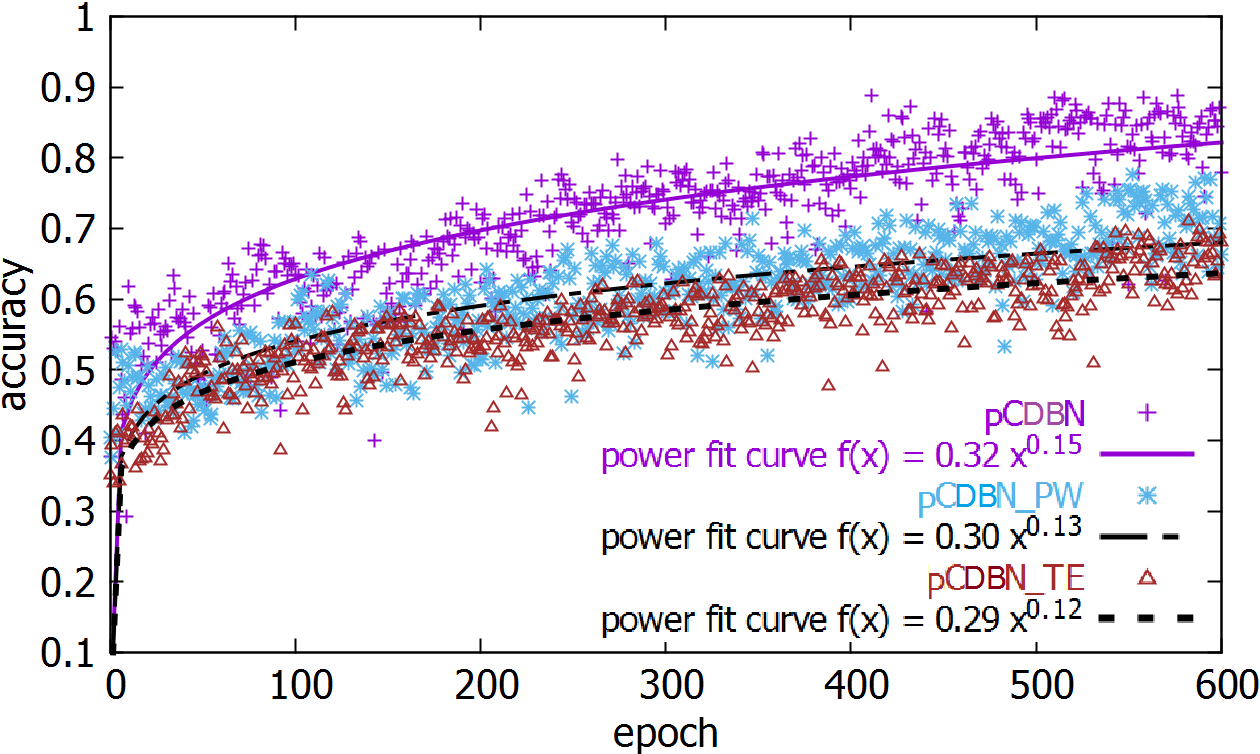} & \includegraphics[width=2.45in]{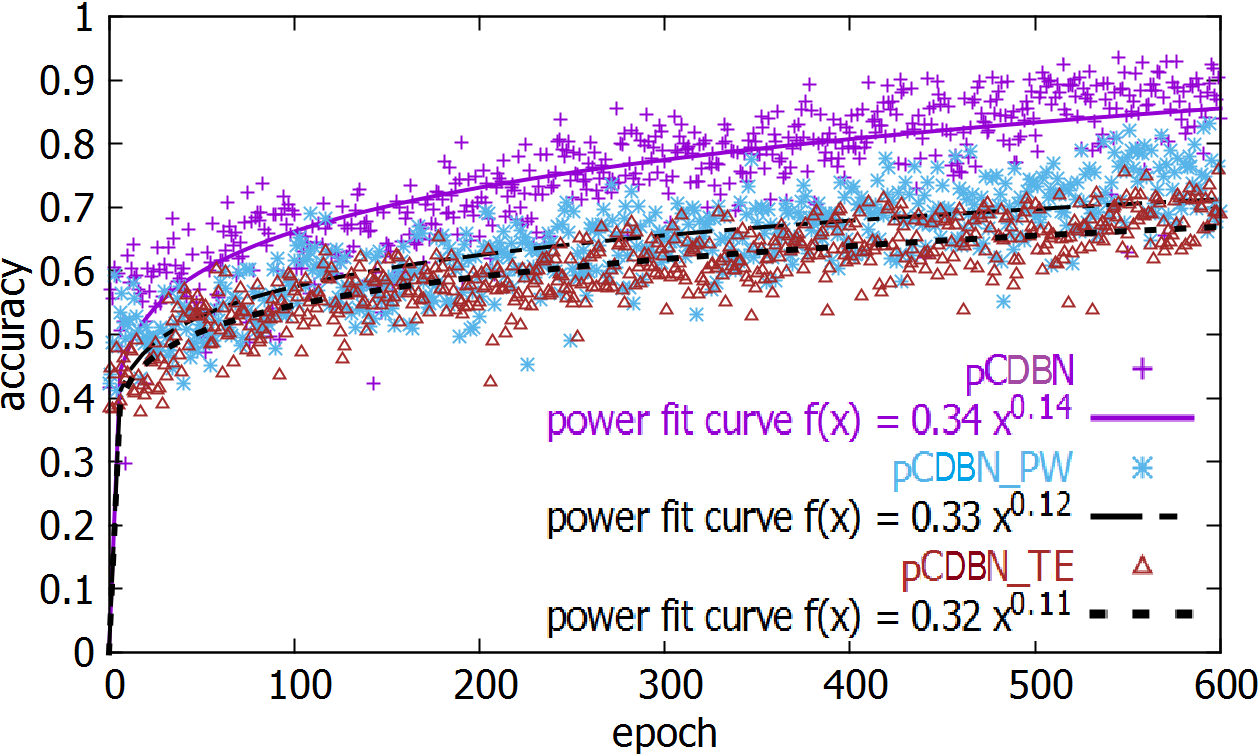} \\ [0.0cm] \mbox{(a) $\epsilon = 0.1$ (Large noise)} & \mbox{(b) $\epsilon = 2$ (Medium noise)} \\ [0.25cm]
\includegraphics[width=2.45in]{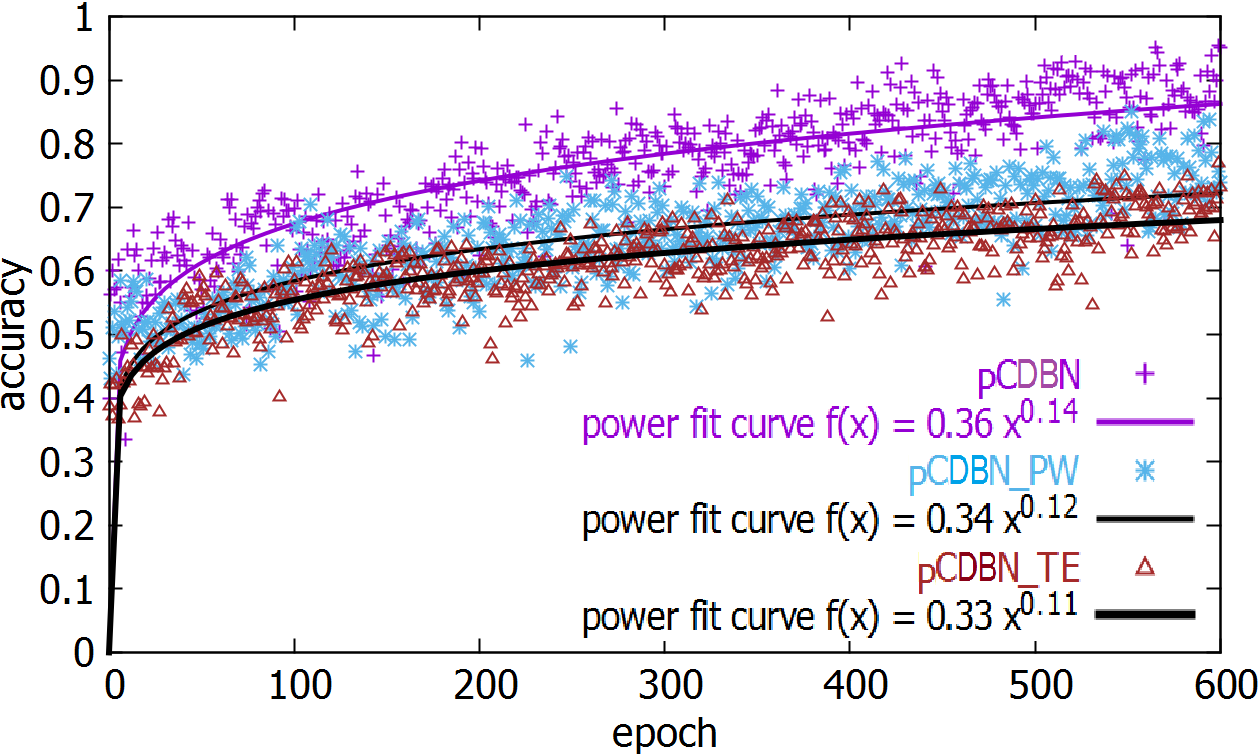} \\ [0.0cm] \mbox{(c) $\epsilon = 8$ (Small noise)}
\end{array}$
\caption{Results of classification accuracy for different noise levels, different approximation approaches, and the number of training epochs.} 
\label{epoch}
\end{figure*}

\subsubsection{Human Behavior Classification} 
In this experiment, we aim to examine: \textbf{(1)} The robustness of our approach when it is trained with a large number of epochs at different noise levels; and \textbf{(2)} The effectiveness of different approximation approaches, including Chebyshev, Taylor, and Piecewise approximations. Our experiment setting is as follows:

We consider every pair $(u, t)$ is a data point. Given $t$ is a week, we have, in total, 9,652 data points (254 users $\times$ 38 weeks). We randomly select 10\% data points as a testing set, and the remaining data points are used as a training set. At each training step, the model is trained with 111 randomly selected data points, i.e., batch size $= 111$. To avoid the imbalance in the data, each training batch consists of a balanced number of data samples from different data classes. With this technique, data points in the under-represented class can be incidentally sampled more than the others \cite{ImbalancedClasses}. The model is used to classify the statuses of all the users given their features. In this experiment, we compare our model with state-of-the-art polynomial approximation approaches in digital implementations, including truncated Taylor Expansion: $\sigma(x) = \tanh x \approx x - \frac{x^3}{3} + \frac{2x^5}{15}$ \cite{DBLPLeeJ98,Miroslav:2012} (\textbf{pCDBN\_TE}), and linear piecewise approximation: $\sigma(x) \approx c_1 x + c_2$ \cite{Armato} (\textbf{pCDBN\_PW}). Other baseline models, i.e., dPAH and SctRBM, cannot be directly applied to this task; so, we do not include them in this experiment.
%The difference between pC2NN\_TE, pC2NN\_PW and our model is that each of those uses a different approximation approach, i.e., Taylor Expansion and linear piecewise. 

%\subsubsection{Competitive Models.}
%We compare our dPC2NN with three types of models: \textbf{(1)} Deep learning models for human behavior prediction. Our original conditional convolutional neural network (\textbf{C2NN}) and \textbf{SctRBM} \cite{KangLi2014,PhanAsonam2015} are competitive models. None of these models enforces $\epsilon$-differential privacy; \textbf{(2)} \textit{Deep Private Auto-Encoder} (\textbf{dPAH}) \cite{Phan0WD16}, which is the state-of-the-art deep learning model under differential privacy for human behavior prediction. The dPAH model outperforms general methods for regression analysis under $\epsilon$-differential privacy, i.e., functional mechanism \cite{zhang2012functional}, DPME \cite{conf/nips/Lei11}, and filter-priority \cite{cormode2011personal}. In this paper, we only compare our model with the dPAH; \textbf{(3)} State-of-the-art polynomial approximation approaches in digital implementations, including truncated series expansion: $\sigma(x) = \tanh x \approx x - \frac{x^3}{3} + \frac{2x^5}{15}$ \cite{DBLPLeeJ98,Miroslav:2012} (\textbf{dPC2NN\_TSE}), and linear piecewise approximation: $\sigma(x) \approx c_1 x + c_2$ \cite{Armato} (\textbf{dPC2NN\_PW}).

$\bullet$ Fig. \ref{epoch} shows classification accuracies for different levels of privacy budget $\epsilon$. Each plot illustrates the evolution of the testing accuracy of each algorithm and its power fit curve as a function of the number of epochs. After 600 epochs, our pCDBN can achieve 88\% with $\epsilon = 0.1$, 92\% with $\epsilon = 2$, and 94\% with $\epsilon = 8$. In addition, our model outperforms baseline approaches, i.e., pCDBN\_TE and pCDBN\_PW, and the results are statistically significant ($p = 4.4293e$-$07$, performed by paired t-test). One of the important observations we acquire from this result is that: The Chebyshev polynomial approximation is more effective than the competitive approaches in preserving differential privacy in convolutional deep belief networks. One of the reasons is that Chebyshev polynomial approximation incurs fewer errors than the other two approaches \cite{Harper2012,Miroslav:2012}. Similar to Layer-wise Relevance Propagation \cite{bach-plos15}, the approximation errors will propagate across neural layers. Therefore, the smaller the error, the more accurate the models will be. 

Note that our observations (i.e., data points) in the YesiWell data are not strictly independent. Therefore, the simple use of paired t-test may not give rigorous conclusions. However, the very small p-values under the paired t-test can still indicate the significant improvement of our approach over baselines.

%In addition, it is attractive that privacy budget consumption is not affected by the number of training epochs. This is sufficient to guarantee that our approach has the ability to work with large datasets without consuming excessive privacy budget.

\begin{figure}[t]
\raggedleft
$\begin{array}{c@{\hspace{0.03in}}c@{\hspace{0.03in}}c}
\includegraphics[width=2.4in]{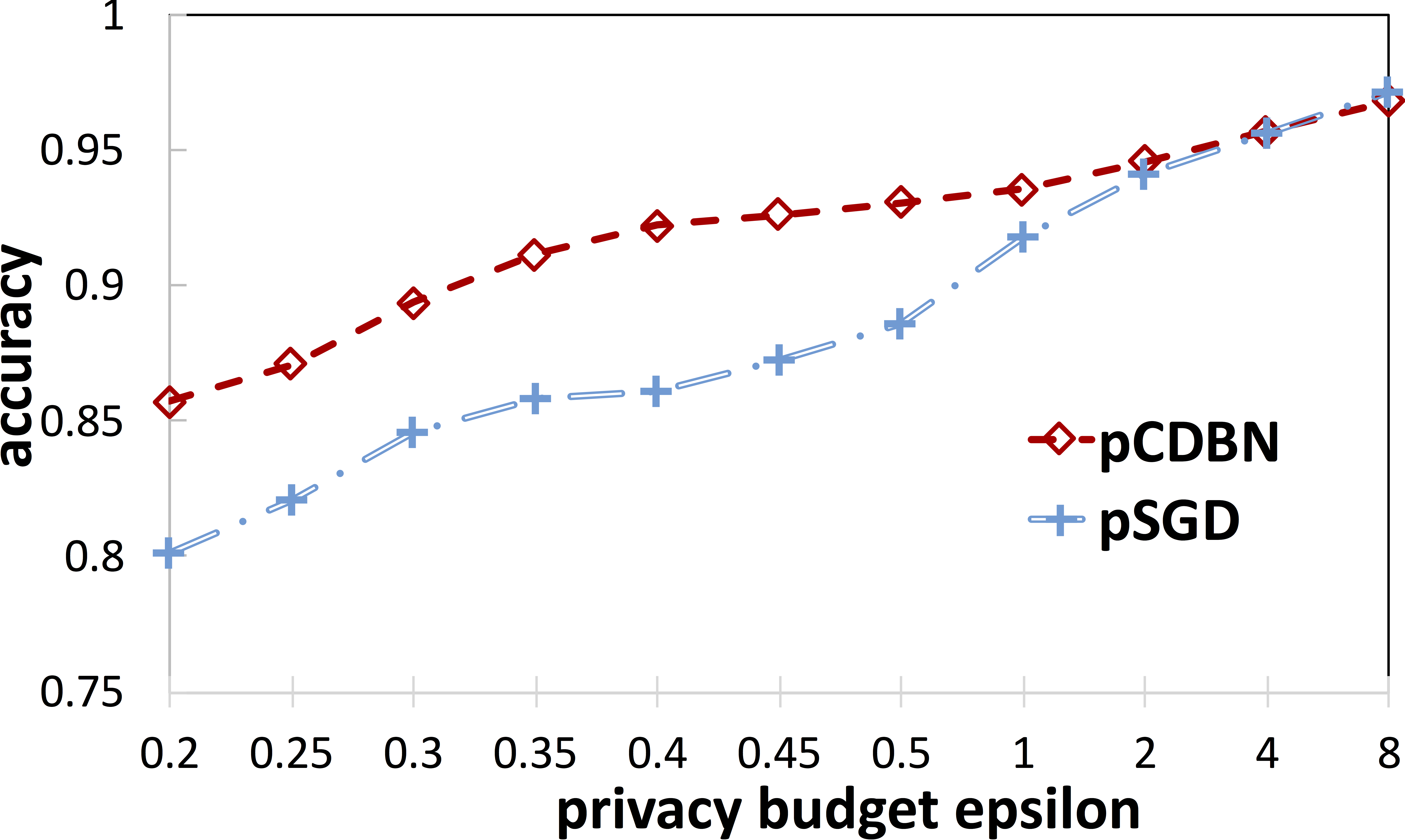} & \includegraphics[width=2.4in]{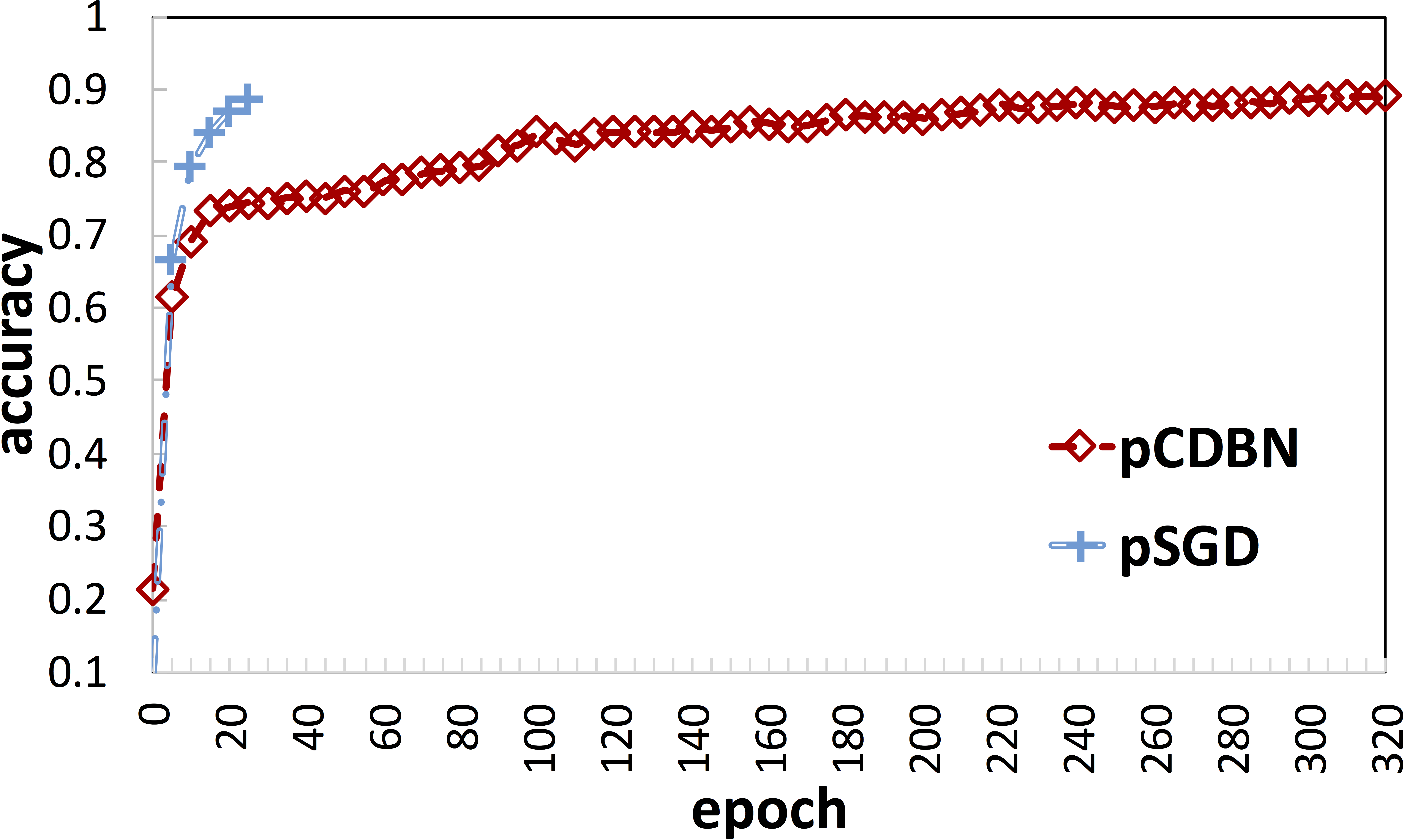} \\ [0.0cm] \mbox{(a) accuracy vs. $\epsilon$} & \mbox{(b) $\epsilon = 0.5$ (large noise)}
\end{array}$
\caption{Accuracy for different noise levels on the MNIST dataset.} 
\label{MNIST2}
\end{figure}

\begin{figure}[t]
\centering
\includegraphics[width=2.8in]{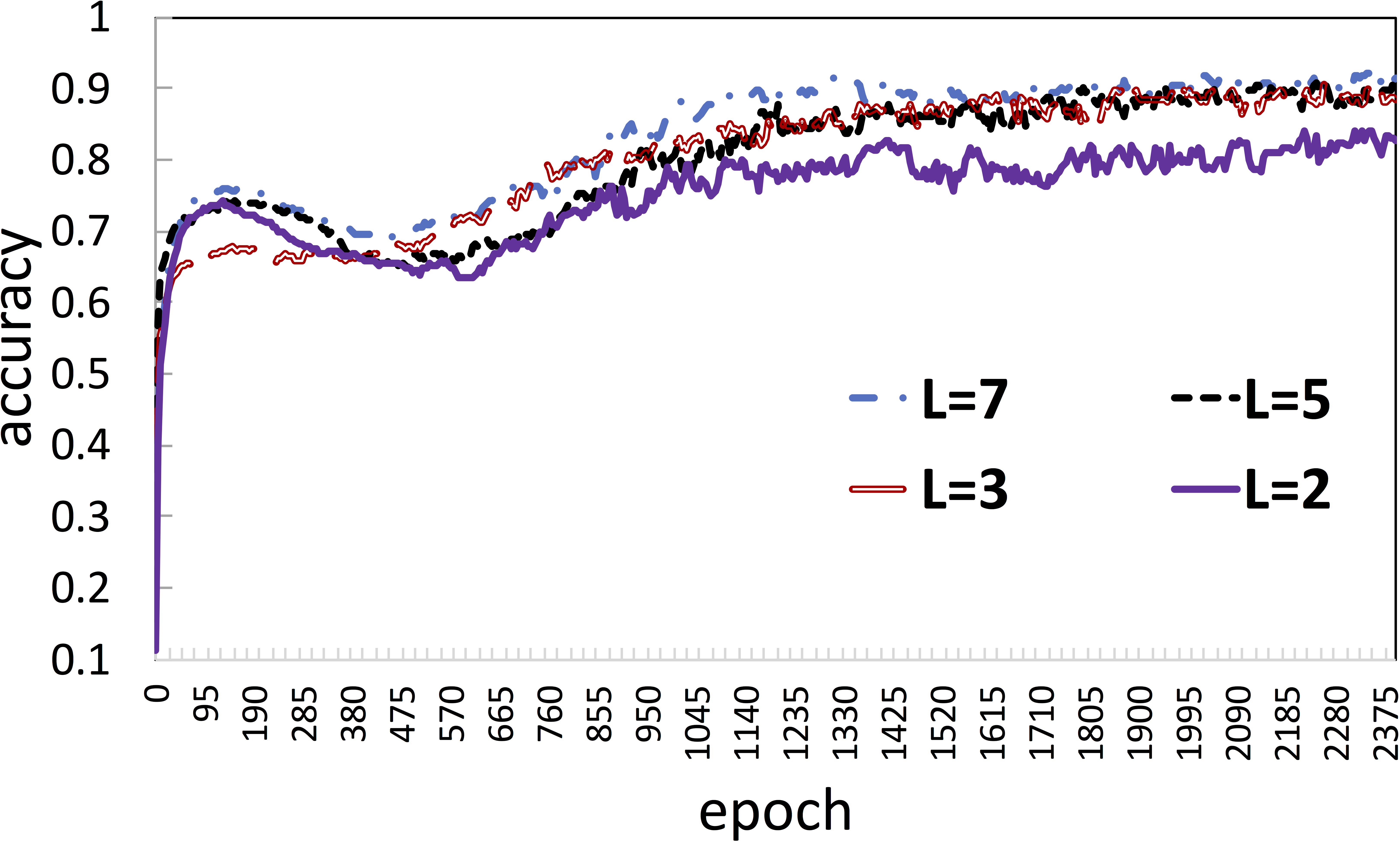}
\caption{The impact of different values of $L$ on the MNIST dataset.} 
\label{L_Impact}
\end{figure}

\subsection{Handwriting Digit Recognition}
To further demonstrate the ability to work on large-scale datasets, we conducted additional experiments on the well-known MNIST dataset \cite{Lecun726791}. The MNIST database of handwritten digits consists of 60,000 training examples, and a test set of 10,000 examples \cite{Lecun726791}. Each example is a 28 $\times$ 28 size gray-level image. The MNIST dataset is completely balanced, with 6,000 images for each category, with 10 categories in total.

We compare our model with the private stochastic gradient descent algorithm, denoted \textbf{pSGD}, from Albadi et al. \cite{Abadi}. The pSGD is the state-of-the-art algorithm in preserving differential privacy in deep learning. pSGD is an advanced version of \cite{ShokriVitaly2015}; therefore, there is no need to include the work proposed by Shokri et al. \cite{ShokriVitaly2015} in our experiments. The two approaches, i.e., our proposed algorithm and the pSGD, are built on the same structure of a convolutional deep belief network. As in prior work \cite{Abadi}, two convolution layers, one with 32 features and one with 64 features, and each hidden neuron which connects with a 5x5 unit patch are applied. On top of the convolution layers, there are a fully-connected layer with 25 units, and a softmax of 10 classes (corresponding to the 10 digits) with cross-entropy loss. 

$\bullet$ Fig. \ref{MNIST2}a illustrates the prediction accuracies of each algorithm as a function of the privacy budget $\epsilon$. We can see that our model pCDBN outperforms the pSGD in terms of prediction accuracies with small values of the privacy budget $\epsilon$, i.e., $\epsilon \leq 1.0$. This is a substantial result, since smaller values of $\epsilon$ enforce a stronger privacy guarantee of the model. With higher values of $\epsilon$ ($ > 1.0$), i.e., small injected noise, the two models converge to similar prediction accuracies.

$\bullet$ Fig. \ref{MNIST2}b demonstrates the benefit of being independent of the number of training epochs in consuming the privacy budget of our mechanism. In this experiment, $\epsilon$ is set to 0.5, i.e., large injected noise. The pSGD achieves higher prediction accuracies after using a small number of training epochs, i.e., 88.59\% after 25 epochs, compared with the pCDBN. More epochs cannot be used to train the pSGD, since it will violate the privacy protection guarantee. Meanwhile, our model, the pCDBN, can be trained with an unlimited number of epochs. After a certain number of training epochs, i.e., 2,400 epochs, the pCDBN outperforms the pSGD in terms of prediction accuracy, with 93.08\% compared with 88.59\%. %Our pCDBN model can reach \%93.08 after 2,400 epochs.

Our experimental results clearly show the ability to work with large-scale datasets using our mechanism. In addition, it is significant to be independent of the number of training epochs in consuming privacy budget $\epsilon$. Our mechanism is the first of its kind offering this distinctive ability. 

$\bullet$ \textbf{The Impact of Polynomial Degree $L$}. Fig. \ref{L_Impact} shows the prediction accuracies of our model by using different values of $L$ on the MNIST dataset \cite{Lecun726791}. After a certain number of training epochs, it is clear that the impact of $L$ is not significant when $L$ is larger than or equal to 3. In fact, the models with $L \geq 3$ converge to similar prediction accuracies after 2,400 training epochs. The difference is notable with small numbers of training epochs. With $L$ larger than 7, the prediction accuracies are very much the same. Therefore we did not show them in Fig. \ref{L_Impact}. Our observation can be used as a gold standard in selecting $L$ when approximating energy functions based on Chebyshev polynomials.

$\bullet$ \textbf{Computational Performance}. Given the MNIST dataset, it takes an average of 1,476 seconds to train our model, after 2,400 epochs, by using a GPU (NVIDIA GTX TITAN X, 12 GB RAM with 3,072 CUDA cores). Meanwhile, training the pSGD is faster than our model, since only a small number of training epochs is needed to train the pSGD. On average, training the pSGD takes 122 seconds, after 25 training epochs. For the YesiWell dataset, training our pCDBN model takes an average of 2,910 seconds, after 600 epochs, compared with 2,141 seconds of the dPAH model.

\section{Conclusions and Discussions}
In this paper, we propose a novel framework for developing convolutional deep belief networks under differential privacy. Our approach conducts both sensitivity analysis and noise insertion on the energy-based objective functions. Distinctive characteristics offered by our model include: \textbf{(1)} It is totally independent of the number of training epochs in consuming privacy budget; \textbf{(2)} It has the potential to be applied in typical energy-based deep neural networks; \textbf{(3)} Non-linear activation functions, which are continuously differentiable (Stone-Weierstrass Theorem \cite{WalterRudin}) and satisfy the Riemann-integrable condition, e.g., tanh, arctan, sigmoid, softsign, sinusoid, sinc, Gaussian, etc. \cite{Activation}, can be applied; and \textbf{(4)} It has the ability to work with large-scale datasets. %It is totally independent of the data size in consuming privacy budget. That guarantees the ability to apply our framework on large datasets. In addition, our framework can be applied given different activation functions, which are continuously differentiable (Stone-Weierstrass Theorem \cite{WalterRudin}) and satisfy the Riemann-integrable condition, e.g., tanh, arctan, sigmoid, softsign, sinusoid, sinc, Gaussian, etc. \cite{Activation}. 
With these fundamental abilities, our framework could significantly improve the applicability of differential privacy preservation in deep learning. To illustrate the effectiveness of our framework, we propose a novel model based on our private convolutional deep belief network (pCDBN), for human behavior modeling. Experimental evaluations on a health social network, YesiWell data, and a handwriting digit dataset, MNIST data, validate our theoretical results and the effectiveness of our approach. 

In future work, it is worthwhile to study how we might be able to extract private information from deep neural networks. We will also examine potential approaches to preserve differential privacy in more complex deep learning models, such as Long Short-Term Memory (LSTM) \cite{DBLP:HochreiterS97}. Another open direction is how to adapt our framework to multiparty computational settings, in which multiple parties can jointly train a deep learning model under differential privacy. Innovative multiparty computational protocols for deep learning under differential privacy must have the ability to work with large-scale datasets. 

In principle, our mechanism can be applied on rectified linear units (ReLUs) \cite{glorot2011deep} and on parametric rectified linear units (PReLUs) \cite{DBLP:HeZR015}. The main difference is that we do not need to approximate the energy function. This is because the energy function is a polynomial function when applying ReLU units. However, we need to add a local response normalization (LRN) layer \cite{krizhevsky2012imagenet} to bound the values of hidden neurons. This is a common step when dealing with ReLU units. The implementation of this layer and ReLU units under differential privacy is an exciting opportunity for other researchers in future work. %We will also rigorously examine our approach when applying other activation functions, such as rectified linear units (ReLUs) \cite{glorot2011deep} and parametric rectified linear units (PReLUs) \cite{DBLP:HeZR015}. We also agree that there may be other approximation approaches which may be better in preserving differential privacy in deep neural networks. This is an open question that needs to be answered by both research and practice communities in the future.

Another challenging problem is identifying the exact risk of re-identification/re-construction of the data under differential privacy. In \cite{DBLP:conf/kdd/LeeC12}, the authors proposed differential identifiability to link individual identifiability to $\epsilon$ differential privacy. However, this is still a non-trivial question. A fancy solution is to design innovative approaches to reconstruct original models from noisy deep neural networks. Then, one could use the original models to infer sensitive information in the training data. However, how to reconstruct the original models from differentially private deep neural networks is an open question. Of course, it is very challenging and will require a significant effort of the whole community to answer.

\begin{acknowledgements}
This work is supported by the NIH grant R01GM103309 to the SMASH project. Wu is also supported by NSF grant 1502273 and 1523115. Dou is also supported by NSF grant 1118050. We thank Xiao Xiao and Rebeca Sacks for their contributions. 
\end{acknowledgements}

\bibliographystyle{spbasic}
\bibliography{dou,prop,thesis,paea,all,related,sigproc,bib-xintao,sigproc2}

\appendix

\section{Corrections of the Paper} 
There was a mistake in terms of model configurations for MNIST data reported in our original version. The number of hidden neurons and epochs are updated in this correction version. Moreover, the code release\footnote{\url{https://github.com/tensorflow/models/tree/master/research/differential_privacy}} of the differentially private Stochastic Gradient Descent algorithm (\textbf{pSGD}) \cite{Abadi} is used in this version. The pSGD algorithm is significantly improved in terms of accuracy, and the computation of the privacy budget is also more accurate than our pSGD implementation used in our original version. The experimental results of our algorithms and the pSGD are updated accordingly for the MNIST data.

%The differences in this correction of our paper can be summarized as follows: There was a mistake in terms of model configurations reported in our original version. The number of hidden neurons and epochs are updated in this correction version. The number of hidden neurons and epochs are changed only for MNIST data. The experimental results of our algorithms and the differentially private Stochastic Gradient Descent algorithm (\textbf{pSGD}) \cite{Abadi} are updated as well, given the MNIST dataset. In fact, the code release of the pSGD algorithm\footnote{\scriptsize{\url{https://github.com/tensorflow/models/tree/master/research/differential_privacy}}} is used in this version. The pSGD algorithm is significantly improved in terms of accuracy, and the computation of the privacy budget $\epsilon$ is also more accurate. 

\end{document}